\def\year{2021}\relax
\documentclass[letterpaper]{article} 

\usepackage{aaai21}

\usepackage{xcolor}

\usepackage{times}  
\usepackage{helvet} 
\usepackage{courier}  
\usepackage[hyphens]{url}  
\usepackage{graphicx} 
\usepackage{natbib}  
\usepackage{caption} 
\usepackage{amsfonts}       
\usepackage{nicefrac}       
\usepackage{microtype}      
\usepackage{enumitem}       
\usepackage[normalem]{ulem}
\usepackage{graphicx}
\usepackage{amsmath,amssymb}
\usepackage{mathtools}
\usepackage{ntheorem}
\usepackage{comment}
\usepackage{algorithm}
\usepackage{algpseudocode}
\DeclareMathOperator*{\argmax}{argmax} 
\newcommand{\qed}{\hfill \ensuremath{\Box}}
\newcommand\inv[1]{#1\raisebox{0.85ex}{$\scriptscriptstyle-\!1$}}

\newtheorem{lemma}{Lemma}
\newtheorem{remark}{Remark}
\newtheorem{corollary}{Corollary}
\newtheorem{theorem}{Theorem}

\newtheorem*{proof}{Proof}

\newcommand{\rev}[1]{\textcolor{black}{#1}}
\title{Reward-Biased Maximum Likelihood Estimation for Linear Stochastic Bandits}

\author {
    Yu-Heng Hung,\textsuperscript{\rm 1}
    Ping-Chun Hsieh,\textsuperscript{\rm 1}
    Xi Liu, \textsuperscript{\rm 2} 
    P. R. Kumar \textsuperscript{\rm 2}\\
}
\affiliations {
    \textsuperscript{\rm 1} National Chiao Tung University \\
    \textsuperscript{\rm 2} Texas A\&M University \\
    j5464654.cs08g@nctu.edu.tw, pinghsieh@nctu.edu.tw, xiliu.tamu@gmail.com, prk@tamu.edu
}
\begin{document}

\maketitle

\begin{abstract}
Modifying the reward-biased maximum likelihood method originally proposed in the adaptive control literature, we propose novel learning algorithms to handle the explore-exploit trade-off in linear bandits problems as well as generalized linear bandits problems. We develop novel index policies that we prove achieve order-optimality, and show that they achieve empirical performance competitive with the state-of-the-art benchmark methods in extensive experiments. \rev{The new policies achieve this with low computation time per pull for linear bandits, and thereby resulting in both favorable regret as well as computational efficiency.} 
\end{abstract}

\section{Introduction}
\label{section:intro}

The problem of decision making for an unknown dynamic system, called stochastic adaptive control \cite{kumar1985survey,kumar1986stochastic}, was examined in the control theory community beginning in the 1950s. It was recognized early on by Feldbaum \cite{feldbaum1960dual1,feldbaum1960dual2} that control played a dual role, that of exciting a system to learn its dynamics, as well as satisfactorily regulating its behavior, therefore dubbed as the problem of ``dual control.''
This leads to a central problem of identifiability: As the controller begins to converge, it ceases to learn about the behavior of the system to other control actions. This issue was quantified by Borkar and Varaiya \cite{borkar1979adaptive} within the setting of adaptive control of Markov chains. Consider a stochastic system with a state-space $X$, control or action set $U$, modelled as a controlled Markov chain with transition probabilities $\mbox{Prob}(x(t+1)=j \vert x(t)=i, u(t)=u) = p(i,j;u,\theta_*)$ dependent on an unknown parameter $\theta_*$ lying in a known set $\Theta$, where $x(t)$ is the state of the system at time step $t$, and $u(t)$ is the action taken at that time. Given a one-step reward function $r(i,u)$, let $\phi: X \times \Theta \rightarrow U$ denote the optimal stationary control law as a function of $\theta \in \Theta$ for the long-term average reward problem: $\max \frac{1}{T} \sum_{t=0}^{T-1} r(x(t),u(t))$, i.e., $u(t)=\phi(x(t),\theta)$ is the optimal action to take if the true parameter is $\theta$.
Since $\theta_*$ is unknown, consider a ``certainty-equivalent'' approach: At each time step $t$, let $\widehat{\theta}_{\text{ML}}(t) \in \argmax_{\theta \in \Theta} \sum_{s=0}^{t-1} \log p(x(s),x(s+1),u(s),\theta)$ denote the Maximum Likelihood (ML) estimate of $\theta_*$, with ties broken according to any fixed priority order. Then apply the action $u(t)=\phi(x(t), \widehat{\theta}_{\text{ML}}(t))$ to the system. It was shown in \cite{kumar1982new} that under an irreducibility assumption, the parameter estimates $\widehat{\theta}_{\text{ML}}(t)$ converge to a random limit $\check{\theta}$ satisfying \begin{equation}
p(i,j,\phi(i, \check{\theta}),\check{\theta}) = p(i,j,\phi(i, \check{\theta}),\theta_*) \quad \quad  \forall i,j \in X. \label{eq closed-loop identification}
\end{equation}
That is, the closed-loop transition probabilities under the control law
$\phi(\cdot, \check{\theta})$ are correctly
determined. 
However, the resulting
feedback control law $\phi(\cdot,\check{\theta})$ need not be optimal for the true parameter $\theta_*$.

A key observation that permitted a breakthrough on this problem was 
made by Kumar and Becker \cite{kumar1982new}. Denote by $J(\phi, \theta)$ the long-term average reward incurred when the stationary control law $\phi$ is used if the true parameter is $\theta$, and by $J(\theta) := \mbox{Max}_\phi J(\phi, \theta)$ the optimal long-term average reward attainable when the parameter is $\theta$.
Then,
\begin{equation}
J(\check{\theta}) \overset{(a)}{=} J(\phi(\cdot,\check{\theta}), \check{\theta}) \overset{(b)}{=} J(\phi(\cdot,\check{\theta}), \theta_*) \overset{(c)}{\leq} J(\theta_*). \label{eq reward inequality}
\end{equation}
where the key equality $(b)$ that the long-term reward under $\phi(\cdot, \check{\theta})$ is the same under the parameters $\check{\theta}$ and $\theta_*$ follows from the equivalence of the closed-loop transition probabilities
(\ref{eq closed-loop identification}), while $(a)$ and $(c)$ hold trivially since $\phi(\cdot,\check{\theta})$ is optimal for $\check{\theta}$, but is not necessarily optimal for $\theta_*$.
Therefore
the maximum likelihood estimator is biased in favor of parameters with \emph{smaller} reward. 
To counteract this bias, \cite{kumar1982new} proposed delicately biasing the ML parameter estimation criterion in the reverse way in favor of parameters with \emph{larger} reward by adding a term $\alpha(t) J(\theta)$ to the log-likelihood,
with $\alpha(t)>0$, $\alpha(t) \to + \infty$, 
and $\frac{\alpha(t)}{t} \to 0$. This results in the \emph{Reward-Biased ML Estimate} (RBMLE):
\begin{equation}
\begin{split}
& \widehat{\theta}_{\text{RBMLE}}(t) \in \\ 
& \argmax_{\theta \in \Theta} \Big\{\alpha(t) J(\theta)+ \sum_{s=0}^{t-1} \log p(x(s),x(s+1),u(s),\theta) \Big\}.
\end{split}
\end{equation}
This modification is delicate since $\alpha(t)=o(t)$, and therefore retains the ability of the ML estimate to estimate the closed-loop transition probabilities, i.e., (\ref{eq closed-loop identification}) continues to hold, for any ``frequent'' limit point $\check{\theta}$ (i.e., that which occurs as a limit along a sequence with positive density in the integers). Hence the bias $J(\check{\theta)} \leq J(\theta_*)$ of
(\ref{eq reward inequality}) continues to hold.
However, since $\alpha(t) \to +\infty$, the bias in favor of parameters with larger rewards ensures that 
\rev{\begin{equation}
J(\check{\theta}) \geq J(\theta_*), \label{eq reverse inequality}
\end{equation}
as shown in \cite[Lemma 4]{kumar1982new}}.
From (\ref{eq reward inequality}) and (\ref{eq reverse inequality}) it follows that $J(\phi(,\cdot,\check{\theta}),\theta_*) = J(\theta_*)$,
whence $\phi(\cdot, \check{\theta})$ is optimal for the unknown $\theta_*$.

The RBMLE method holds potential as a general-purpose method for the learning of dynamic systems. 
However, its analysis was confined to \textit{long-term average optimality}, which only
assures that the regret is $o(t)$. 
Pre-dating the Upper Confidence Bound (UCB) method of Lai and Robbins \cite{lai1985asymptotically}, RBMLE has largely remained unexplored vis-\`a-vis its finite-time performance as well as empirical performance on contemporary problems. Motivated by this, there has been recent interest in revisiting the RBMLE. Recently, its regret performance has been established
for classical multi-armed bandits for the exponential family of measures \cite{liu2019bandit}. However, classical bandits do not allow the incorporation of ``context,'' which is important in various applications \cite{li2010contextual,lu2010contextual,chapelle2011empirical,li2016collaborative,tewari2017ads}. 
\rev{Therefore, the design and the proofs in \cite{liu2019bandit} cannot directly apply to the more structured contextual bandit model.}
In this paper, we examine the RBMLE method both for linear contextual bandits as well as a more general class of generalized linear bandits. 
Linear bandits and their variants have been popular models for abstracting the sequential decision making in various applications, such as recommender systems \cite{li2010contextual} and medical treatment \cite{tewari2017ads}.

This paper extends the RBMLE principle and obtains simple index policies for linear contextual bandits as well as their generalizations that have provable order-optimal {finite-time regret} performance as well as empirical performance competitive with the best currently available.
\rev{The main contributions of this paper are as follows:}
\begin{itemize}[leftmargin=*]
    \item We extend the RBMLE principle to linear contextual bandits by proposing a specific type of reward-bias term. 
    We introduce into RBMLE the modification of using a Gaussian pseudo-likelihood function, both for usage in situations where the distribution of the rewards is unknown, as well as to derive simple index policies. \rev{Different from the popular UCB-based policies, whose indices usually consist of two components: a maximum likelihood estimator and a confidence interval, RBMLE directly incorporates a reward-bias term into the log-likelihood function to guide the exploration instead of using concentration inequalities. The derived RBMLE index is thereby different from the existing indices for linear bandits.}
    \item \rev{We show that the so modified RBMLE index attains a regret bound of $\smash[b]{\mathcal{O}(\sqrt{T}\log T)}$, which is order-optimal (within a logarithmic factor) for general, possibly non-parametric, sub-Gaussian rewards. To the best of our knowledge, this is the first provable finite-time regret guarantee of the classic RBMLE principle for contextual bandits. 
    This bound shaves a factor of $\smash[b]{\mathcal{O}(\sqrt{T^{\epsilon}})}$ from Thompson Sampling (LinTS) \cite{agrawal2013thompson}, a factor of $\smash[b]{\mathcal{O}(\sqrt{\log T})}$ from \cite{chu2011contextual}, and a factor of $\smash[b]{\mathcal{O}(\sqrt{\log^3 T})}$ from Gaussian Process Upper Confidence Bound (GPUCB) with linear kernels \cite{srinivas2010gaussian}, and achieves the same regret bound as the Information Directed Sampling (IDS) \cite{kirschner2018information}.}
    \item We extend the techniques to the generalized linear models and show that the same regret bound of $\smash[b]{\mathcal{O}(\sqrt{T}\log T)}$ can still be attained in the general case. {\color{black}This shaves a factor of $\sqrt{\log T}$ from \cite{filippi2010parametric}}, and achieves the same regret bound as UCB-GLM in \cite{li2017provably}.
    \item \rev{We also conduct extensive experiments to demonstrate that the proposed RBMLE achieves an empirical regret competitive with the state-of-the-art benchmark methods while being efficient in terms of computation time. Notably, the regret performance of RBMLE is the most robust across different sample paths. The results validate that the proposed algorithm enjoys favorable regret as well as computation time.}
\end{itemize}

\section{Problem Setup}
\label{section:problem}
We consider the stochastic contextual bandit problem with $K < +\infty$ arms, possibly large.
At the beginning of each decision time $t\in\mathbb{N}$, a $d$-dimensional context vector $x_{t,a}\in\mathbb{R}^{d}$, with $\lVert x_{t,a}\rVert\leq 1$, is revealed to the learner, for each arm $a \in [K]$.
The contexts $\{x_{t,a}\}$ are generated by an adaptive adversary, which determines them in an arbitrary way based on the history of all the contexts and rewards.
Given the contexts, the learner selects an arm $a_t\in [K]$ and obtains the corresponding reward $r_t$, which is conditionally independent of all the other rewards in the past given the context $\{x_{t,a_t}\}$.
We define (i) $x_t := x_{t, a_t}$,
(ii) $X_t$ as the $(t-1)\times d$ matrix in which the $s$-th row is $x_s^\intercal$, for all $s\in [t-1]$,
(iii) $R_t:=(r_1,\cdots, r_{t-1})^{\intercal}$ row vector of the observed rewards up to time $t-1$, and (iv)
$\mathcal{F}_t=(x_1,a_1,r_1,\cdots,x_t)$ denotes the $\sigma$-algebra of all the causal information available right before $r_t$ is observed.
We assume that the rewards are linearly realizable, i.e., there exists an unknown parameter $\theta_*\in\mathbb{R}^{d}$ with
$\lVert \theta_*\rVert_2\leq 1$, and a known, strictly increasing \textit{link function} $\mu:\mathbb{R}\rightarrow \mathbb{R}$ such that $\mathbb{E}[r_t|{\color{black}\mathcal{F}_t}]=\mu(\theta_*^\intercal x_t)$.
{\color{black} We assume that $\mu$ is continuously differentiable, with
 its derivative  $\mu'$ having a supremum $L_{\mu}$, and an infimum $\kappa_{\mu}>0$}.\footnote{A further discussion about this assumption is in Appendix \ref{appendix:discussion about assumption of GLM-RBMLE}.} {\color{black}We call this the \emph{generalized linear bandit} problem.} 
 

Let $a_t^*:=\arg\max_{1\leq i\leq K} \theta_*^\intercal x_{t,i}$ be an arm that yields the largest 
{\color{black}conditional} expected reward 
$\mathbb{E}[r_t|{\color{black}\mathcal{F}_t}]$
at time $t$ (with ties broken arbitrarily), {\color{black} and
$x_t^* := x_{t,a_t^*}$}. 
The objective of the learner is to maximize its total 
over a finite time horizon $T$, i.e.,
the learner aims to minimize the \textit{total {\color{black}conditional} expected {\color{black}pseudo-}regret},
{\color{black} which we shall refer to simply as the ``cumulative regret,"} defined as
\begin{equation}
    \mathcal{R}(T):=\sum_{t=1}^{T} \mu(\theta_*^{\intercal} x_t^*) - \mu(\theta_*^{\intercal} x_t).
\end{equation}

We call the problem a \emph{standard} linear bandits problem if
(i) the reward is  $r_t=\theta_*^\intercal x_t +\varepsilon_t$, (ii) $\varepsilon_t$ is a noise with $\mathbb{E}[\varepsilon_t|x_t]=0$, and (iii)
the rewards are conditionally $\sigma$-sub-Gaussian, i.e.,
\begin{equation}
    \mathbb{E}[\exp(\rho\varepsilon_t)|\mathcal{F}_t]\leq \exp\big(\frac{\rho^2\sigma^2}{2}\big).
\end{equation}
Wlog, we assume $\sigma=1$. For standard linear bandits the link function $\mu$ is an identity and
$\kappa_\mu=1$.

\section{RBMLE for Standard Linear Bandits}
\label{section:linear}
We begin with the derivation of the RBMLE index and its regret analysis for linear contextual bandits. 

\subsection{Index Derivation for Standard Linear Bandits}
\label{section:linear:index}
Let $\ell(\mathcal{F}_t;\theta)$ denote the log-likelihood of the historical observations when the true parameter is $\theta$.
Let $\lambda$ be a positive constant. At each $t$, the learner takes the following two steps. 
\begin{enumerate}
\item Let $\bar{\theta}_t = \argmax\limits_\theta \big\{ \ell(\mathcal{F}_t;\theta)+\alpha(t) \max\limits_{a \in [K]} \theta^{\intercal}x_{t,a}-\frac{\lambda}{2}{\lVert \theta \rVert}^2_2 \big\}$.
\item Choose any arm $a_t$ that maximizes        $\bar{\theta}_t^{\intercal}x_{t,a}
 $.

\end{enumerate}
The term $\alpha(t) \max_{1\leq a\leq K} \theta^{\intercal}x_{t,a}$ is the reward-bias. 
A modification to the RBMLE is the additional quadratic regularization term $\frac{\lambda}{2}{\lVert \theta \rVert}_2^2$, {\`a} la ridge regression. 
Wlog, we assume that $\lambda\geq 1$.

The above strategy can be
simplified to an \emph{index strategy}.
Define the index of an arm $a$ at time $t$ by
\begin{equation}
    {\cal{I}}_{t,a}:=\max_{\theta} \Big\{ \ell(\mathcal{F}_t;\theta)+\alpha(t)\cdot \theta^{\intercal}x_{t,a}-\frac{\lambda}{2}{\lVert \theta \rVert}^2_2 \Big\}, \label{eq:LinRBMLE index via tilde Theta t,a}
\end{equation}
and simply choose an arm $a_t$ that has maximum index. The indexability proof is in Appendix A. 

To derive indices, it is necessary to know what
the log-likelihood $\ell(\mathcal{F}_t;\theta)$ is.
However, in practice, the true distribution of the noise $\varepsilon_t$ is unknown to the learner or it may not even follow any parametric distribution. 
We employ the Gaussian density function as a surrogate:
\begin{equation}
\ell(\mathcal{F}_t;\theta)=-\frac{1}{2}\sum_{s=1}^{t-1}(\theta^{\intercal}x_{s}-r_{s})^2-\frac{t-1}{2}\log(2\pi). \label{eq Gaussian Likelihood}
\end{equation}
Hence $\bar{\theta}_t$ is any maximizer of  
$ \Big\{
    -\sum_{s=1}^{t-1}(\theta^{\intercal}x_{s}-r_{s})^2+2\alpha(t)\cdot \max_{1 \leq a\leq K} \theta^\intercal x_{t,a}-\lambda{\lVert \theta \rVert}^2_2\Big\}
$.

It is shown in Section \ref{section:linear:regret} that despite the likelihood misspecification, the index derived from the Gaussian density achieves the same regret bound for general non-parametric sub-Gaussian rewards.

The LinRBMLE index has the following explicit form, as proved
in Appendix \ref{appendix:corollary:RBMLE index for linear bandits}:
\begin{corollary}
\label{corollary:RBMLE index for linear bandits}
\normalfont For the Gaussian likelihood (\ref{eq Gaussian Likelihood}), there is a unique maximizer of (\ref{eq:LinRBMLE index via tilde Theta t,a}) for every arm $a$,
\begin{equation}
    \bar{\theta}_{t,a}=\inv{V_{t}}(X_{t}^{\intercal}R_{t}+\alpha(t)x_{t,a}), \label{eq:solution for widetilde theta t,a}
\end{equation}
where $V_t:=X_{t}^{\intercal}X_{t}+\lambda I$.
The arm $a_t$ chosen by the LinRBMLE algorithm is
\begin{align}
        a_t&= \argmax_{1\leq i \leq K} \Big\{ \widehat{\theta}_t^{\intercal} x_{t,i} + \frac{1}{2}\alpha(t)\lVert x_{t,i} \rVert^2_{V_t^{-1}} \Big\},\label{eq:RBMLE index for linear bandits}
\end{align}
where $\widehat{\theta}_t := \inv{V_t}X_{t}^{\intercal} R_{t}$ is the least squares estimate of $\theta_{*}$.
\end{corollary}

We summarize the LinRBMLE algorithm in Algorithm \ref{alg:LinRBMLE}.

\begin{algorithm}[!htbp]
\caption{LinRBMLE Algorithm}
\label{alg:LinRBMLE}
    \begin{algorithmic}[1]
        \State {\bfseries Input:} $\alpha(t)$,  $\lambda$\;
        \State {\bfseries Initialization:} $V_1\leftarrow \lambda I$\;
        \For{$t=1,2,\cdots$}
            \State Observe the contexts $\{x_{t,a}\}$ for all the arms\;
            \State Select the action $a_t=\argmax_{a} \big\{\widehat{\theta}_{t}^{\intercal}x_{t,a}+$
            
            $\frac{1}{2}\alpha(t)\lVert x_{t,a} \rVert^2_{V_t^{-1}}\big\}$ and obtain $r_t$\;
            \State Update $V_{t+1}\leftarrow V_{t}+x_{t,a_t} x_{t,a_t}^{\intercal}$
        \EndFor
    \end{algorithmic}
\end{algorithm}
\begin{remark}
\label{remark:LinRBMLE vs LinUCB}
\normalfont Similar to the well-known LinUCB index $\widehat{\theta}_t^{\intercal} x_{t,i} +\gamma \lVert x_{t,i}\rVert_{\inv{V_t}}$ \cite{li2010contextual}, the LinRBMLE index is also defined as the sum of the least squares estimate and an additional exploration term. 
Despite this high-level resemblance, LinRBMLE has two salient features: \rev{(i) As mentioned in Section \ref{section:intro}, the LinRBMLE index is different from the UCB-based indices as it directly incorporates a reward-bias term into the log-likelihood function to guide the exploration instead of using concentration inequalities;
(ii) Under LinRBMLE, the ratio between the exploration terms of any two arms $i,j$ is ${\lVert x_{t,i} \rVert^2_{V_t^{-1}} }/ \lVert x_{t,j} \rVert^2_{V_t^{-1}}$, which is more contrastive than $\lVert x_{t,i} \rVert_{V_t^{-1}}/\lVert x_{t,j} \rVert_{V_t^{-1}}$ of LinUCB.
With a proper bias term, this design of LinRBMLE implicitly encourages more exploration (since $\lVert x_{t,i} \rVert_{V_t^{-1}}$ is a confidence interval).
As will be seen in Section \ref{section:linear:regret}, with a proper bias term (e.g., $\alpha(t)=\sqrt{t}$), this additional exploration does not sacrifice the regret bound. Moreover, as suggested by the regret statistics in Section \ref{section:experiment}, this design makes LinRBMLE empirically more robust across different sample paths, which is of intrinsic interest. }

\end{remark}


\subsection{Regret Bound for the LinRBMLE Index}
\label{section:linear:regret}
We begin the regret analysis with a bound on the ``immediate" regret $R_t:=\theta_{*}^{\intercal}(x_t^*-x_{t})$.
\begin{lemma}
\label{lemma:decomposed linear regret}
\normalfont
Under the standard linear bandit model, 
\begin{equation}
\begin{split}
    R_t &\leq {\lVert\theta_*-\widehat{\theta}_t\rVert_{V_t}\cdot\lVert x_t^*\rVert_{V_t^{-1}} - \frac{1}{2}\alpha(t)\lVert x_t^*\rVert^2_{V_t^{-1}} } \\
    & + {\lVert\widehat{\theta}_t - \theta_*\rVert_{V_t}\cdot \lVert x_t\rVert_{V_t^{-1}}}
    + {\frac{1}{2}\alpha(t)\lVert x_t\rVert^2_{V_t^{-1}}}.\label{eq:lemma1:0} 
\end{split}
\end{equation}
\end{lemma}
The proof of Lemma \ref{lemma:decomposed linear regret} is in Appendix \ref{appendix:lemma:decomposed linear regret}.
\begin{remark}
\label{remark:proof difference}
\normalfont Lemma \ref{lemma:decomposed linear regret} highlights the main difference between the analysis of the UCB-based algorithms (e.g., \cite{abbasi2011improved,chu2011contextual}) and that of the LinRBMLE algorithm. 
To arrive at a regret upper bound for LinRBMLE, it is required to handle both $\lVert\theta_*^{\intercal}-\widehat{\theta}_t\rVert_{V_t}\cdot\lVert x_t^*\rVert_{V_t^{-1}}$ and $\frac{1}{2}\alpha(t)\lVert x_t^*\rVert^2_{V_t^{-1}}$. While it could be challenging to quantify each individual term, we show in Theorem \ref{theorem:linear regret} that a tight regret upper bound can be obtained by jointly analyzing these two terms. 
\end{remark}
Theorem \ref{theorem:linear regret} below presents the regret bound for the LinRBMLE algorithm; it is proved in Appendix \ref{appendix:theorem:linear regret}. Let
\begin{align}
    G_0(t,\delta) &:={\sigma}\sqrt{{d}\log\left(\frac{\lambda+{t}}{\lambda\delta}\right)}+\lambda^{\frac{1}{2}}, \\
     G_1(t) & :=\sqrt{2d\log\left(\frac{\lambda +t}{d}\right)}\label{eq:GLM-g1}\mbox{  respectively}.
\end{align}
\begin{theorem}
\label{theorem:linear regret}
\normalfont
For the LinRBMLE index (\ref{eq:RBMLE index for linear bandits}), with probability at least $1-\delta$, the cumulative regret satisfies
\begin{equation}
\begin{split}
    \mathcal{R}(T) &= \sum_{t=1}^{T}R_t \leq \big(G_0(T,\delta)\big)^2\cdot \Big(\sum_{t=1}^{T}\frac{1}{2\alpha(t)}\Big) \\
    & +\sqrt{T}G_0(T,\delta)G_1(T)+\frac{1}{2}\alpha(T)\big(G_1(T)\big)^2.
\end{split}
\end{equation}
Consequently, by choosing the bias term $\alpha(t) =\sqrt{t}$, the regret bound is $\mathcal{R}(T)=\mathcal{O}(\sqrt{T}\log T)$.

\rev{
\begin{remark}
\normalfont As mentioned in Section \ref{section:intro}, LinRBMLE achieves a better regret bound than several popular benchmark methods, including LinTS \cite{agrawal2013thompson}, SupLinUCB \cite{chu2011contextual}, and GPUCB with a linear kernel \cite{srinivas2010gaussian}.
Moreover, LinRBMLE achieves the same regret bound as that of IDS \cite{kirschner2018information}, which is one of the most competitive benchmarks. In Section \ref{section:experiment}, we show via simulations that LinRBMLE achieves an empirical regret competitive with IDS while being much more computationally efficient.
LinRBMLE also has the same regret bound as that of LinUCB \cite{abbasi2011improved}. As LinRBMLE addresses exploration in a fundamentally different manner as discussed in Remark \ref{remark:LinRBMLE vs LinUCB}, the corresponding regret proof also differs from those of the UCB-base policies, as highlighted in Remark \ref{remark:proof difference}.  
From the simulations, we further observe that LinRBMLE significantly outperforms LinUCB in terms of both empirical mean regret and regret statistics.
\end{remark}
}

\end{theorem}

\section{RBMLE for Generalized Linear Bandits}
\label{section:generalized}
\subsection{Index Derivation for Generalized Linear Bandits}
\label{section:generalized:index}
For the generalized linear case, as before, let $\bar{\theta}_t$ be any maximizer of 
   $ \big\{ \ell(\mathcal{F}_t;\theta)+\alpha(t)\cdot \max_{1 \leq a \leq K} \theta^{\intercal}x_{t,a}-\frac{\lambda}{2}{\lVert \theta \rVert}^2_2 \big\}$.
However, a major difference vis-{\`a}-vis the {\color{black}standard} linear case is that $L_{\mu}> \kappa_{\mu}$.
To handle this, we incorporate an additional factor $\eta(t)$ that is a positive-valued, strictly increasing function that satisfies $\lim_{t\rightarrow \infty}\eta(t)=\infty$, and choose any arm $a_t$ that maximizes $\big\{ \ell(\mathcal{F}_t;\bar{\theta}_{t,a})+\eta(t)\alpha(t)\cdot \bar{\theta}_{t,a}^{\intercal}x_{t,a}-\frac{\lambda}{2}{\lVert \bar{\theta}_{t,a}\rVert}_2^2  \big\}$.
The regret analysis below suggests that it is sufficient to choose $\eta(t)$ to be slowly increasing, e.g., $\eta(t)=1+\log t$.


Next, we generalize the notion of a {\color{black} surrogate} Gaussian likelihood discussed in Section \ref{section:linear:index} by considering the density functions of the canonical exponential families:
\begin{equation}
    p(r_t|x_t)=\exp(r_t x_t^{\intercal}\theta_{*}-b(x_t^{\intercal} \theta_{*})+c(r_t)),\label{eq:exponential family density}
\end{equation}
where $b(\cdot):\mathbb{R}\rightarrow \mathbb{R}$ is a strictly convex function that satisfies $b'(z)=\mu(z)$, for all $z\in\mathbb{R}$, and $c(\cdot):\mathbb{R}\rightarrow \mathbb{R}$ is the normalization function.
The exponential family consists of a variety of widely used distributions, including binomial, Gaussian, and Poisson distributions.
By the properties of the exponential family, $b'(x_t^{\intercal} \theta_{*})=\mathbb{E}[r_t|x_t]$ and $b''(x_t^{\intercal} \theta_{*})=\mathbb{V}[r_t|x_t]>0$.
By (\ref{eq:tilde Theta t,a}) and the strict convexity of $b(\cdot)$, $\ell(\mathcal{F}_t;\theta)+\alpha(t)\cdot \theta^{\intercal}x_{t,a}$ is strictly concave in $\theta$ and therefore has a unique maximizer. By the first-order sufficient condition, $\bar{\theta}_{t,a} \mbox{ is the unique solution to }$
\begin{equation}
    \sum_{s=1}^{t-1} \big(r_s x_s - \mu(x_s^{\intercal}\bar{\theta}_{t,a})x_s\big) -\lambda \bar{\theta}_{t,a}+ \alpha(t)x_{t,a}=0.\label{eq:tilde theta first-order condition}
\end{equation}

{Note that (\ref{eq:exponential family density}) is used only for index derivation and is not required in the regret analysis in Section \ref{section:generalized:regret}.}
We summarize the resulting GLM-RBMLE algorithm for the generalized linear case in Algorithm \ref{alg:GLM-RBMLE}.

\begin{remark}
\normalfont \rev{The technical reason behind incorporating $\eta(t)$ into GLM-RBMLE is as follows:
As will be seen in (\ref{eq:generalized regret 2-1})-(\ref{eq:generalized regret 2-2}) in Appendix F, the immediate regret $R_t$ is upper bounded by the value of a quadratic function of $\smash[b]{\lVert x_{t,a_t^*}\rVert_{\inv{U}}}$, and this inequality resembles (\ref{eq:lemma1:4}) for the linear case. 
To further bound the RHS of (\ref{eq:generalized regret 2-1}), we need the leading coefficient $\smash[b]{{L_{\mu}^3}/{(2\kappa_{\mu}^2 \eta(t))}}-1$ to be negative.
To ensure this, we propose to set $\eta(t)$ to be a positive, strictly increasing function with $\lim_{t\rightarrow \infty}\eta(t)=\infty$ such that $\smash[b]{{L_{\mu}^3}/{(2\kappa_{\mu}^2 \eta(t))}}<1$ for all sufficiently large $t$.
For the linear case, we can simply let $\eta(t)=1$ since $L_{\mu}=\kappa_{\mu}=1$ and $\smash[b]{{L_{\mu}^3}/{2\kappa_{\mu}^2}}<1$ automatically holds.}
\end{remark}

\begin{algorithm}[!htbp]
\caption{GLM-RBMLE Algorithm}
\label{alg:GLM-RBMLE}
    \begin{algorithmic}[1]
        \State {\bfseries Input:} $\alpha(t)$,  $\lambda$, $\eta(t)$\;
        \For{$t=1,2,\cdots$}
            \State Observe the contexts $\{x_{t,a}\}$ for all the arms\;
            \State Calculate $\bar{\theta}_{t,a}$ for each $a$ by solving
            $\sum_{s=1}^{t-1} \big(r_s x_s -$
            
            $\mu(x_s^{\intercal}\bar{\theta}_{t,a})x_s\big) -\lambda \bar{\theta}_{t,a}+ \alpha(t)x_{t,a}=0$\;
            \State Select the action $a_t=\argmax_{a} \big\{ \ell(\mathcal{F}_t;\bar{\theta}_{t,a})+$
            
            $\eta(t)\alpha(t) \bar{\theta}_{t,a}^{\intercal}x_{t,a}-\frac{\lambda}{2}{\lVert \bar{\theta}_{t,a}\rVert}_2^2  \big\}$ and obtain $r_t$\;
        \EndFor
    \end{algorithmic}
\end{algorithm}

\subsection{Regret Bound for GLM-RBMLE for Generalized Linear Bandits}
\label{section:generalized:regret}
We begin the regret analysis of GLM-RBMLE by introducing the following definitions.  

Define $T_0:=\min\{t\in\mathbb{N}: \frac{L_{\mu}^3}{2\kappa_{\mu}^2\eta(t)}< \frac{1}{2}\}$.
Recall that $G_1(t)$ is defined in (\ref{eq:GLM-g1}). For ease of exposition, we also define the function
\begin{align}
    G_2(t,\delta)&:=\frac{\sigma}{\kappa_{\mu}}\sqrt{\frac{d}{2}\log(1+\frac{2t}{d})+\log\frac{1}{\delta}}.
\end{align}
We also define 
$C_1:={2 L_{\mu}^4}/{k_{\mu}^4}+{1}/{k_{\mu}^2}$, $C_2:={2L_{\mu}^3}/{\kappa_{\mu}^2}+{L_{\mu}}/{\kappa_{\mu}}$, and $C_3:={L_{\mu}^2}/{2}$.
\begin{theorem}
\label{theorem:generalized regret}
\normalfont For the GLM-RBMLE index, 
with probability at least $1-\delta$, 
the cumulative regret satisfies
\begin{equation}
\begin{split}
    \mathcal{R}(T) \leq T_0 & + C_1\alpha(T)\big(G_1(T)\big)^2\\
    & + C_2 \sqrt{T}G_1(T)G_2(T,\delta) \\
    & + C_3 \big(G_2(T,\delta)\big)^2\sum_{t=1}^{T}\frac{1}{\alpha(t)}.\label{eq:formal result of generalized regret bound}
\end{split}
\end{equation}
Therefore, if $\alpha(t)=\Omega(\sqrt{t})$, then $\mathcal{R}(T)=\mathcal{O}(\alpha(T)\log T)$; If $\alpha(t)=\mathcal{O}(\sqrt{t})$, then $\mathcal{R}(T)=\mathcal{O}\big((\sum_{t=1}^{T}\frac{1}{\alpha(t)})\log T\big)$. 
Hence, by choosing
$\alpha(t)=\sqrt{t}$, 
$\mathcal{R}(T)=\mathcal{O}(\sqrt{T}\log T)$.
\end{theorem}

\begin{remark}
\normalfont {\color{black}This bound improves that in \cite{filippi2010parametric} by a $\sqrt{\log T}$ factor and is the same as that of UCB-GLM \cite{li2017provably}.}
\end{remark}


\label{section:analysis}

\section{Numerical Experiments}
\label{section:experiment}
\useunder{\uline}{\ul}{}
 
To evaluate the performance of the proposed RBMLE methods, we conduct a comprehensive empirical comparison with other state-of-the-art methods vis-a-vis three aspects: effectiveness (cumulative regret), efficiency (computation time per decision vs. cumulative regret), and scalability (in number of arms and dimension of contexts). We paid particular attention to fairness of comparison and reproducibility of results. To ensure sample-path sameness for all methods, we compared each method over a pre-prepared dataset containing the context of each arm and the outcomes of pulling each arm over all rounds. Hence, the outcome of pulling an arm is obtained by querying the pre-prepared data instead of calling the random generator and changing its state. A few benchmarks such as LinTS and Variance-based Information Directed Sampling (VIDS) that rely on outcomes of random sampling in each round of decision-making are separately evaluated with the same prepared data and with the same seed. To ensure the reproducibility of experimental results, we set up the seeds for the random number generators at the beginning of each experiment and provide all the codes. 

\rev{To present a comprehensive numerical study similar to \cite{russo2018learning},} the benchmark methods compared include LinUCB \cite{chu2011contextual}, LinTS \cite{agrawal2013thompson}, Bayes-UCB (BUCB) \cite{kaufmann2012bayesian}, GPUCB \cite{srinivas2010gaussian} and its variant GPUCB-Tuned (GPUCBT) \cite{russo2018learning}, Knowledge Gradient (KG) and its variant KG* \cite{ryzhov2010robustness,ryzhov2012knowledge,kaminski2015refined}, and VIDS \cite{russo2018learning}. 
A detailed review of these methods is presented in Section \ref{section:related}. The values of their hyper-parameters are as follows. For LinRBMLE, as suggested by Theorem \ref{theorem:linear regret}, \rev{we choose $\alpha(t)=\sqrt{t}$ without any hyper-parameter tuning, and $\lambda=1$ which is a common choice in ridge regression and is not sensitive to the empirical regret.} We take $\alpha=1$ in LinUCB and $\delta=10^{-5}$ in GPUCB. We tune the parameter $c$ in GPUCBT for each experiment and choose $c=0.9$ that achieves the best performance. We follow the suggestion of \cite{kaufmann2012bayesian} to choose $c=0$ for BUCB. Respecting the restrictions in \cite{agrawal2013thompson}, we take $\delta=0.5$ and $\epsilon=0.9$ in LinTS. In the comparison with IDS and VIDS, we sampled $10^3$ points over the interval $[0,1]$ for $q$ and take $M=10^4$ in sampling (Algorithm 4 and 6 in~\cite{russo2018learning}). In the Bayesian family of benchmark methods (LinTS, BUCB, KG, KG*, GPUCB, GPUCBT, and VIDS), the prior distribution over the unknown parameters $\theta_{*}$ is $\mathcal{N}(0_d,I_d)$. The comparison contains 50 trials of experiments and $T$ rounds in each trial. We consider both contexts, ``static,'' where the context for each arm is fixed in each experiment trial, and ``time-varying,'' where the context for each arm changes from round to round. 

\rev{The procedure for generating the synthetic dataset is as follows: (i) All contexts are drawn randomly from $\mathcal{N}(0_d,10I_d)$ and normalized by their $\ell_2$ norm; (ii) At time $t$, the reward of each arm $i$ is sampled independently from $\mathcal{N}(\mu(\theta^\intercal_*x_{t,i}),1)$. 
In each test case, we consider a fixed $\theta_*$ and randomly generate the contexts, which lead to different mean rewards across the arms.
This scheme for generating the synthetic dataset has been widely adopted in the bandit literature, such as \cite{abbasi2011improved,dumitrascu2018pg,kirschner2018information};
(iii) As IDS-based approaches are known to be time-consuming, we choose $d=3$ as suggested by \cite{kirschner2018information} for the experiments involving regret comparison in order to finish enough simulation steps within a reasonable amount of time. For the scalability experiments, we reduce the number of rounds $T$ to allow the choice of larger $d$'s.
}

\textbf{Effectiveness.} Figure \ref{fig:regret} and Table \ref{table:TimeStatic/ID=2} illustrate the effectiveness of LinRBMLE in terms of cumulative regret. We observe that for both static and time-varying contexts, LinRBMLE achieves performance only slightly worse than the best performing algorithm, which is often GPUCBT or VIDS. However, compared to these two, LinRBMLE has some salient advantages. In contrast to LinRBMLE, GPUCBT has no guaranteed regret bound and requires tuning the hyper-parameter $c$ to establish its outstanding performance. This restricts its applicability if pre-tuning is not possible. Compared to VIDS, the computation time of LinRBMLE is two orders of magnitude smaller, as will be shown in Figure \ref{fig:time all}. As shown in Table \ref{table:TimeStatic/ID=2}, LinRBMLE also exhibits better robustness with an order of magnitude or two smaller std. dev. compared to VIDS and many other benchmark methods. 
In Figure \ref{fig:regret}(a), VIDS appears to have not converged, but a detailed check reveals that this is only because its performance in some trials is much worse than in other trials. 
The robustness is also reflected in variation across problem instances, e.g., the performance of VIDS is worse in the problem of Figure \ref{fig:regret}(b) than in the problem of Figure \ref{fig:regret}(a), while the performance of LinRBMLE is consistent in these two examples.
The robustness of LinRBMLE across different sample paths can be largely attributed to the inclusion of the Reward Bias term $\alpha(t)$ in the index (\ref{eq:RBMLE index for linear bandits}), which encourages more exploration even for those sample paths with small $\lVert x_{t,i}\rVert_{\inv{V_t}}$. It is worth mentioning that the advantage of VIDS compared to other methods is less obvious for time-varying contexts. Experimental results reported in \cite{russo2018learning} are restricted to the static contexts. More statistics of final cumulative regret in Figure \ref{fig:regret} are provided in the appendix.

\textbf{Efficiency.} Figure \ref{fig:time all} presents the averaged cumulative regret versus average computation time per decision. 
We observe that LinRBMLE and GPUCBT have points closest to the origin, signifying small regret simultaneously with small computation time,
and outperform the other methods.


\textbf{Scalability.} Table \ref{table:TimeStatic_computation_time} presents scalability of computation time per decision as $K$ and $d$ are varied. We observe that both LinRBMLE and GPUCBT,
which are often the best among the benchmark methods have low computation time as well as better scaling when $d$ or $K$ are increased. \rev{LinRBMLE is slightly better than LinUCB in terms of computation time under various $K$ and $d$ since the calculation of LinUCB index requires an additional square-root operation.} Such scalability is important for big data applications such as recommender and advertising systems.

For generalized linear bandits, a similar study on effectiveness, efficiency, and scalability for GLM-RBMLE and popular benchmark methods is detailed in Appendix \ref{appendix:experiments}.

\begin{figure*}[!h]
$\begin{array}{c c c c}
    \multicolumn{1}{l}{\mbox{\bf }} & \multicolumn{1}{l}{\mbox{\bf }} & \multicolumn{1}{l}{\mbox{\bf }} & \multicolumn{1}{l}{\mbox{\bf }}\\ 
    \hspace{-3mm} \scalebox{0.27}{\includegraphics[width=\textwidth]{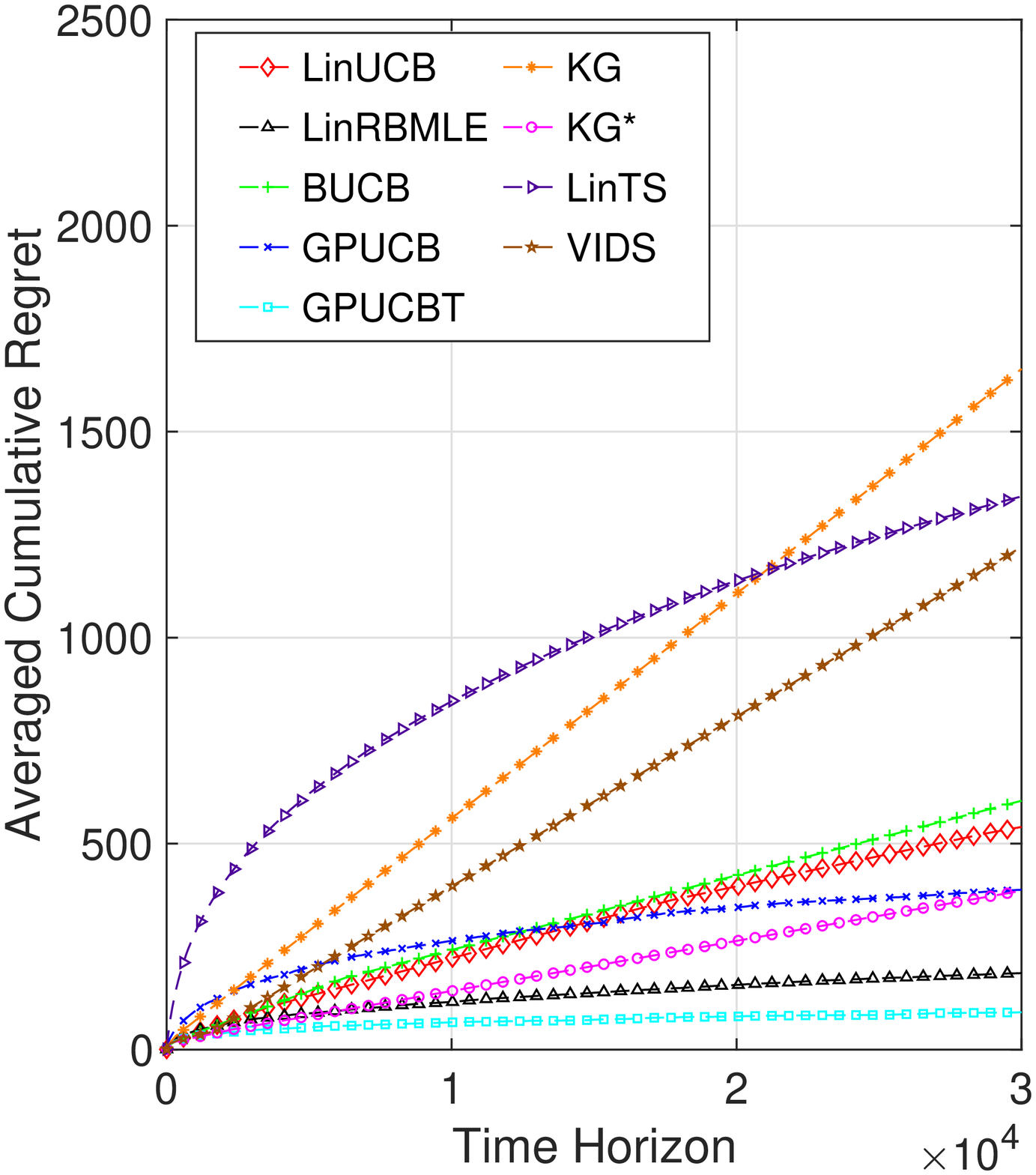}}  \label{fig:TimeStatic_ID_2} & \hspace{-6mm} \scalebox{0.27}{\includegraphics[width=\textwidth]{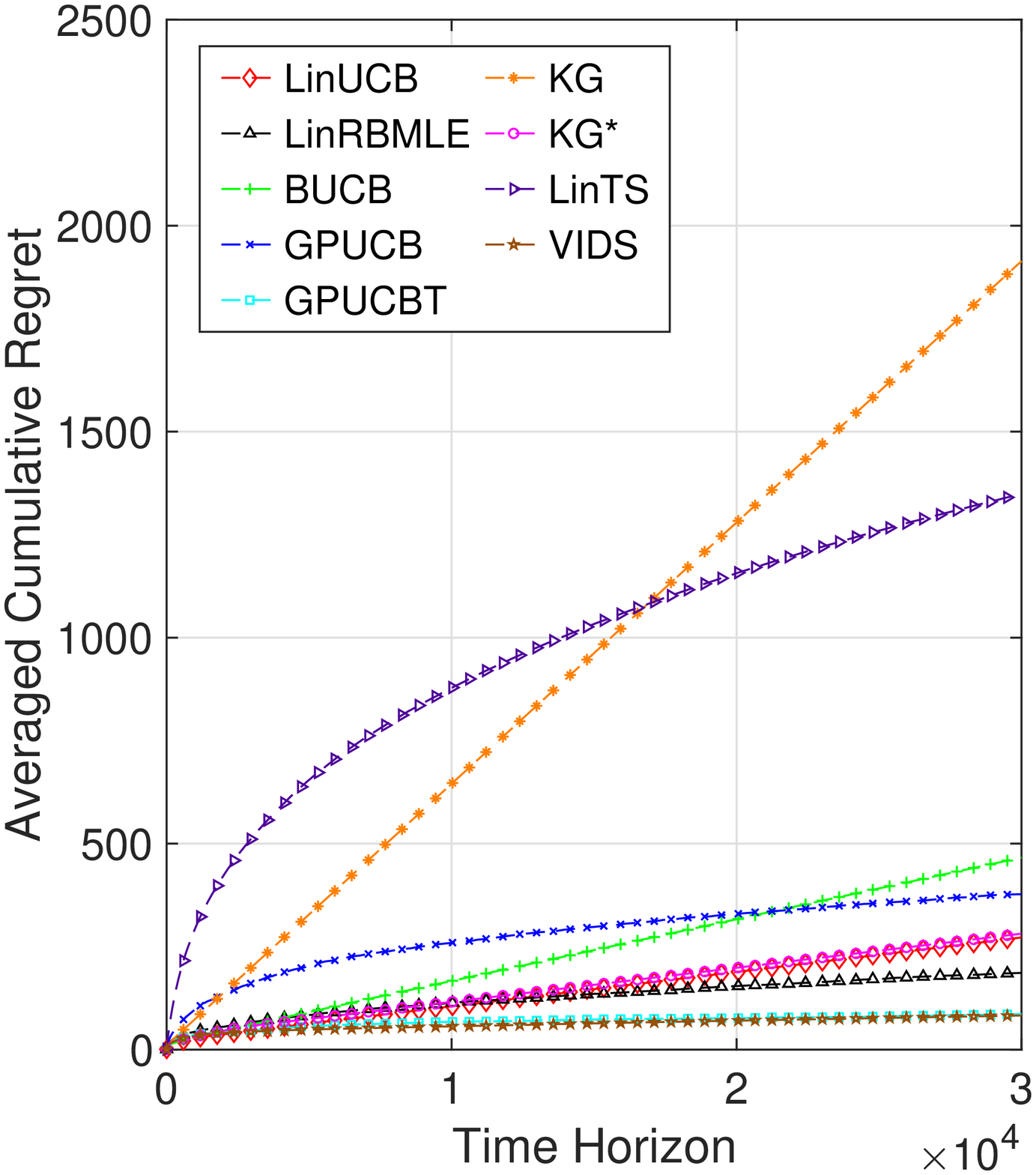}} \label{fig:TimeStatic_ID_9} & \hspace{-6mm} \scalebox{0.27}{\includegraphics[width=\textwidth]{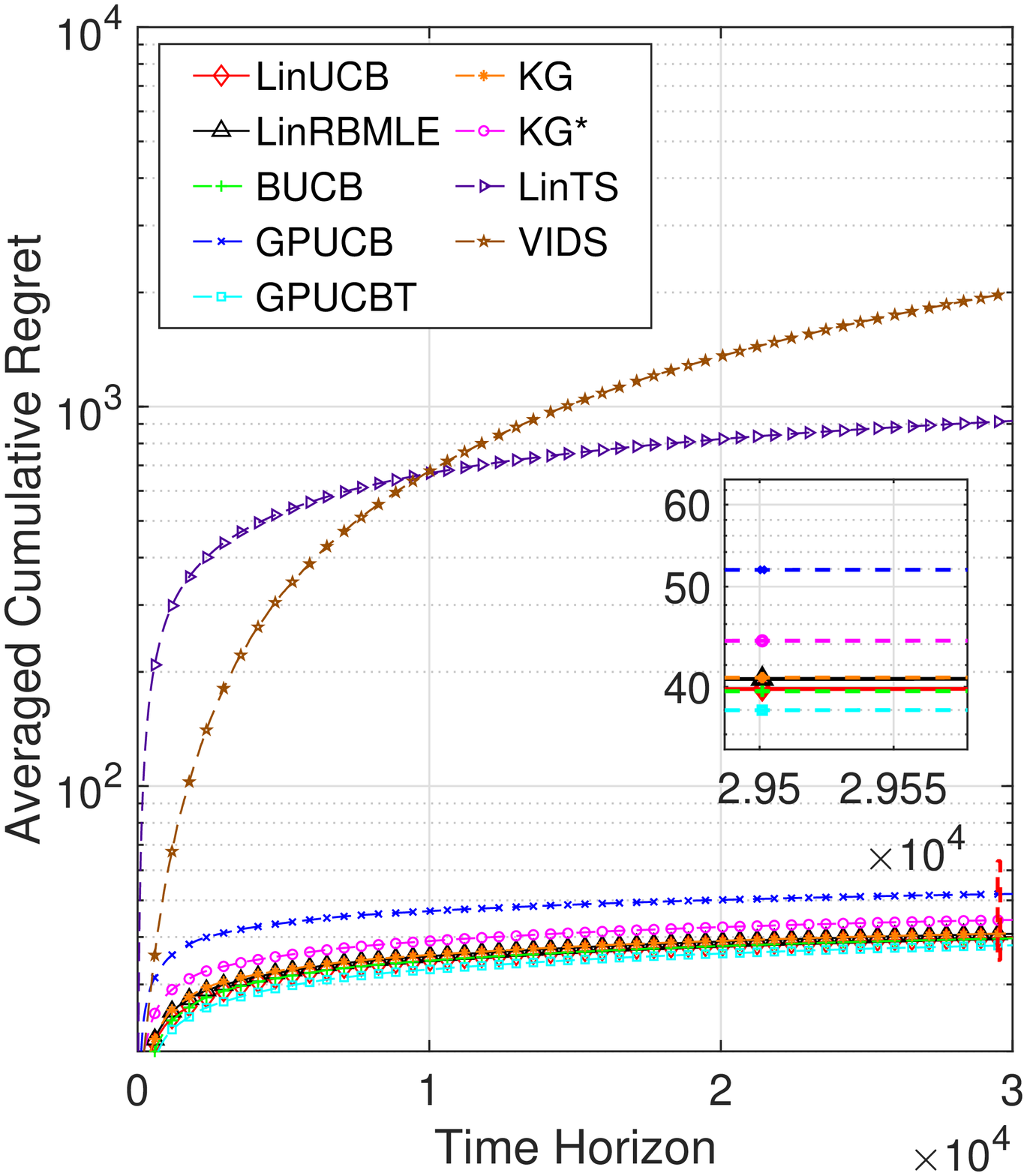}} \label{fig:TimeVary_ID_2}& \hspace{-6mm}
    \scalebox{0.27}{\includegraphics[width=\textwidth]{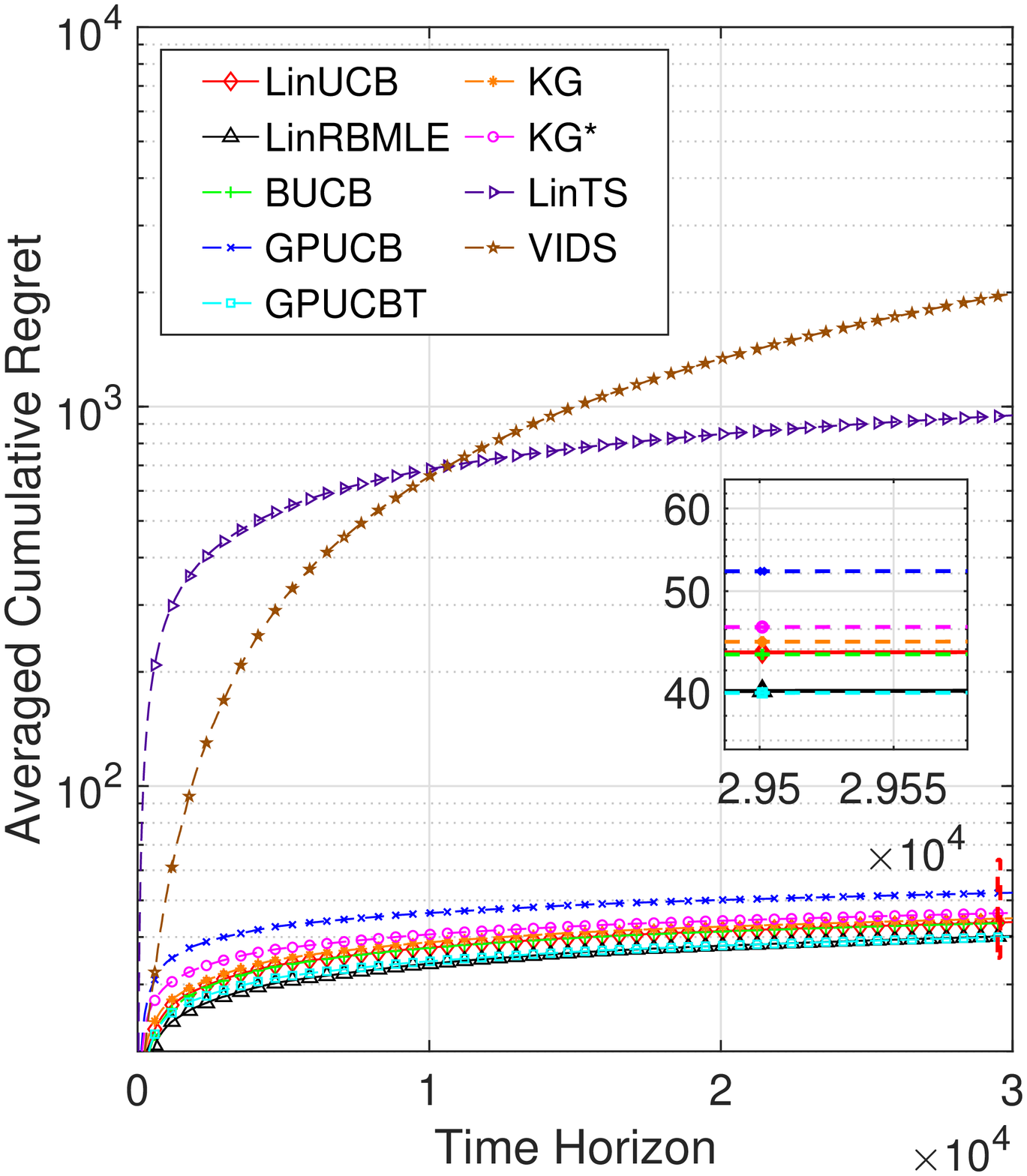}} \label{fig:TimeVary_ID_9}\\ [0.0cm]
    \mbox{\small (a)} & \hspace{-2mm} \mbox{\small (b)} & \hspace{-2mm} \mbox{\small (c)} & \hspace{-2mm} \mbox{\small (d)}\\[-0.2cm]
\end{array}$
\caption{Cumulative regret averaged over 50 trials with $T=3\times 10^4$ and $K=10$: (a) and (b) are under static contexts; (c) and (d) are under time-varying contexts; (a) and (c) are with $\theta_* = (-0.3, 0.5, 0.8)$; (b) and (d) are with with $\theta_* = (-0.7, -0.6, 0.1)$.}
\label{fig:regret}
\end{figure*}

\begin{table*}[!h]
\footnotesize
\begin{center}
\begin{tabular}{|c|c|c|c|c|c|c|c|c|c|}
\hline
\textbf{Alg.} & \textbf{RBMLE} & \textbf{LinUCB} & \textbf{BUCB} & \textbf{GPUCB} & \textbf{GPUCBT} & \textbf{KG} & \textbf{KG*} & \textbf{LinTS} & \textbf{VIDS} \\ \hline
Mean & {\ul \textbf{1.86}} & 5.41 & 6.04 & 3.88 & {\ul \textbf{0.90}} & 16.52 & 3.86 & 13.43 & 12.20 \\ \hline
Std.Dev & {\ul \textbf{0.42}} & 14.87 & 11.78 & 1.19 & {\ul \textbf{0.53}} & 26.68 & 10.46 & 2.20 & 74.66 \\ \hline
Q.10 & 1.45 & {\ul \textbf{0.04}} & 0.07 & 2.30 & 0.32 & {\ul \textbf{0.03}} & 0.07 & 10.83 & 0.15 \\ \hline
Q.25 & 1.62 & {\ul \textbf{0.07}} & 0.10 & 3.01 & 0.59 & {\ul \textbf{0.05}} & 0.10 & 12.44 & 0.29 \\ \hline
Q.50 & 1.79 & {\ul \textbf{0.15}} & {\ul \textbf{0.14}} & 3.78 & 0.79 & 0.18 & 0.18 & 13.58 & 0.45 \\ \hline
Q.75 & 1.96 & 1.00 & 1.30 & 4.56 & 1.09 & 23.83 & {\ul \textbf{0.34}} & 14.25 & {\ul \textbf{0.79}} \\ \hline
Q.90 & {\ul \textbf{2.31}} & 19.34 & 23.00 & 5.74 & {\ul \textbf{1.66}} & 64.89 & 18.94 & 15.73 & 2.38 \\ \hline
Q.95 & {\ul \textbf{2.75}} & 30.47 & 36.31 & 5.91 & {\ul \textbf{1.98}} & 75.96 & 27.18 & 16.78 & 9.40 \\ \hline
\end{tabular}
\caption{Statistics of the final cumulative regret in Figure \ref{fig:regret}(a). The best and the second-best are highlighted. `Q' and ``Std.Dev'' stand for quantile and standard deviation of the total cumulative regret over 50 trails, respectively. All the values displayed here are scaled by 0.01 for more compact notations.}
\label{table:TimeStatic/ID=2}
\end{center}
\vspace*{-1 cm}
\end{table*} 
\begin{figure*}[!h]
\centering
$\begin{array}{c c c c}
    \multicolumn{1}{l}{\mbox{\bf }} & \multicolumn{1}{l}{\mbox{\bf }} & \multicolumn{1}{l}{\mbox{\bf }} & \multicolumn{1}{l}{\mbox{\bf }}  \\
    \hspace{-3mm} \scalebox{0.255}{\includegraphics[width=\textwidth]{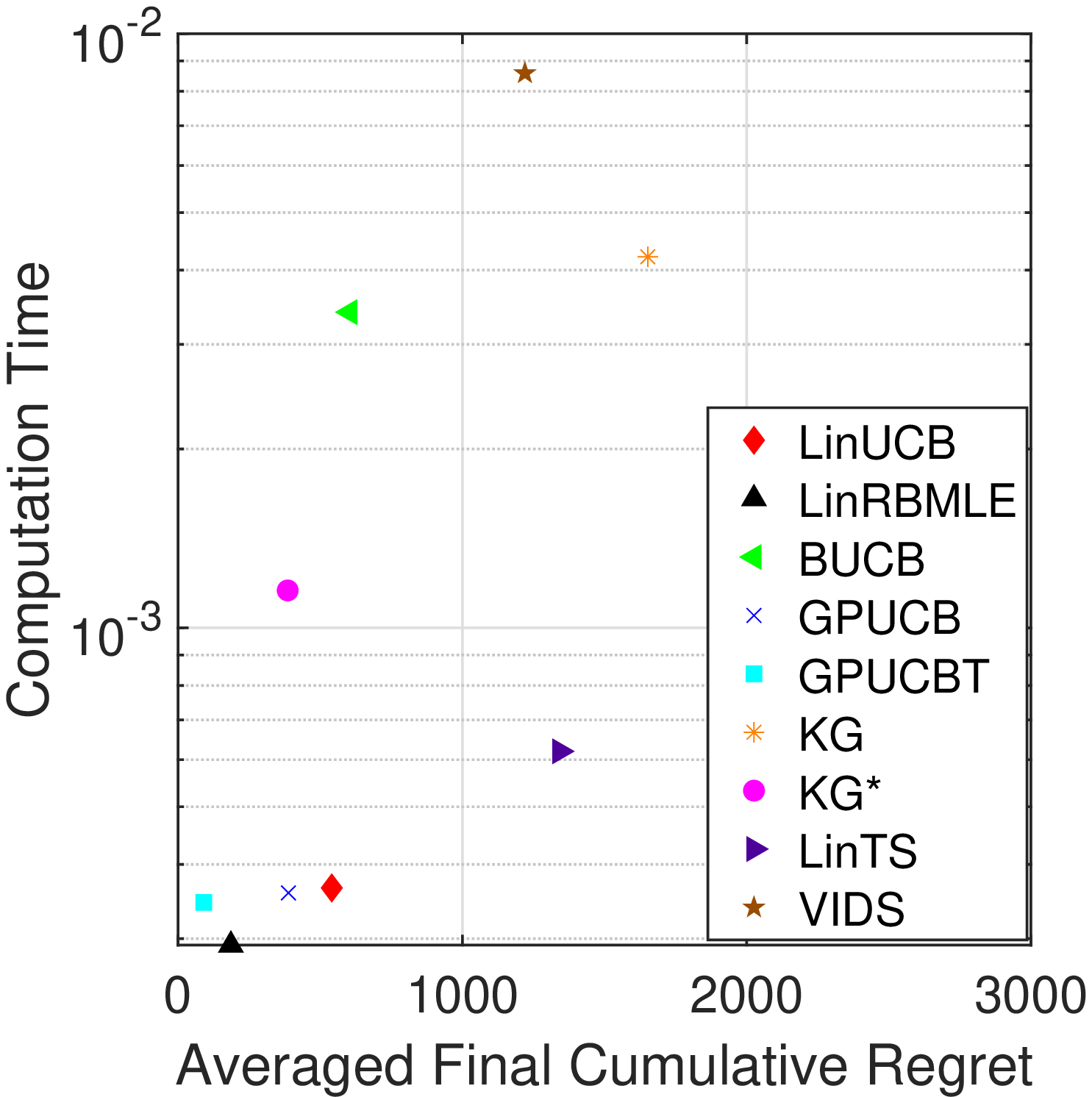}} \label{fig:time all TimeStatic ID=2} & 
    \hspace{-3mm} \scalebox{0.255}{\includegraphics[width=\textwidth]{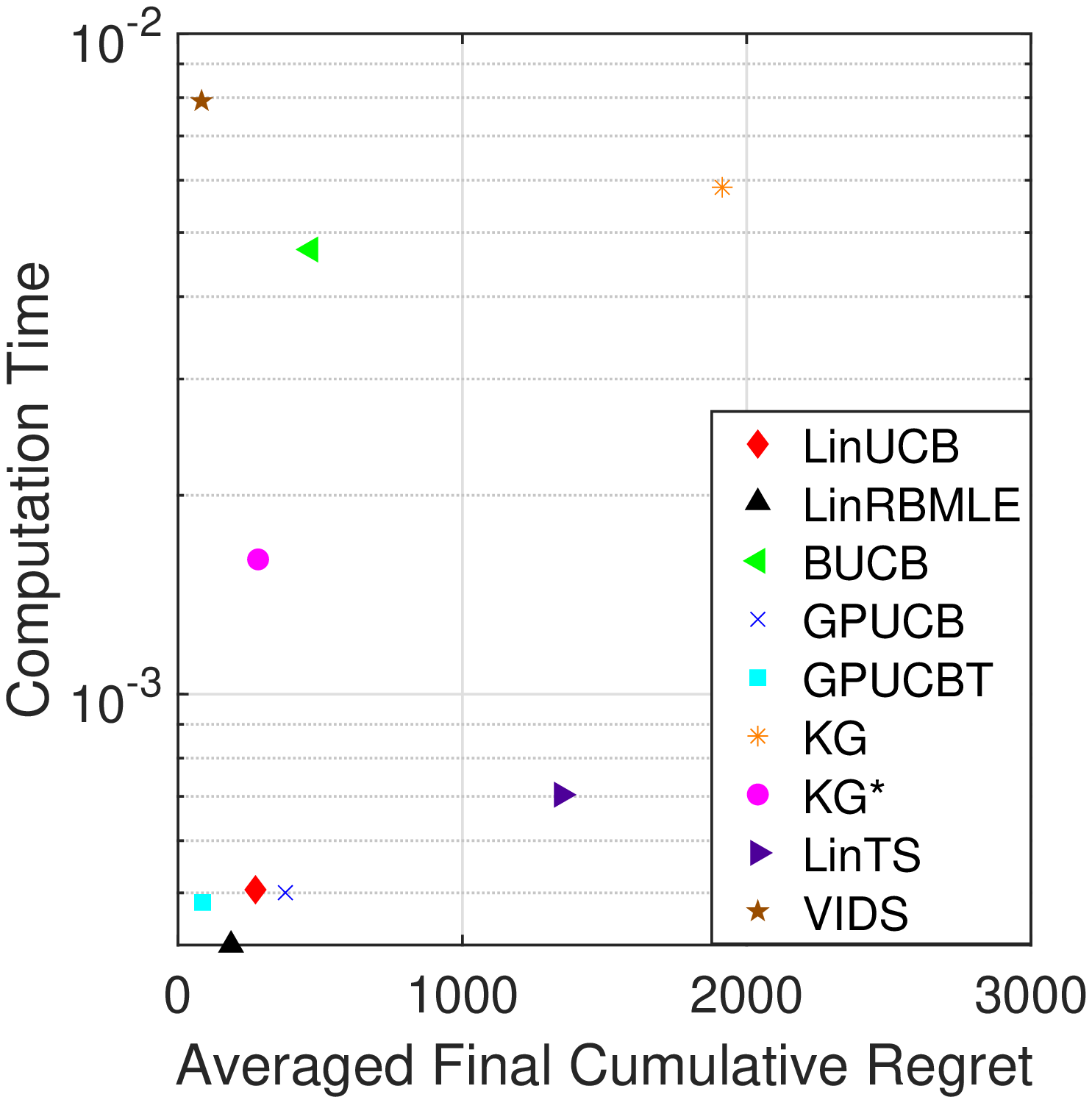}} \label{fig:time all TimeStatic ID=9} & 
    \hspace{-3mm}
    \scalebox{0.255}{\includegraphics[width=\textwidth]{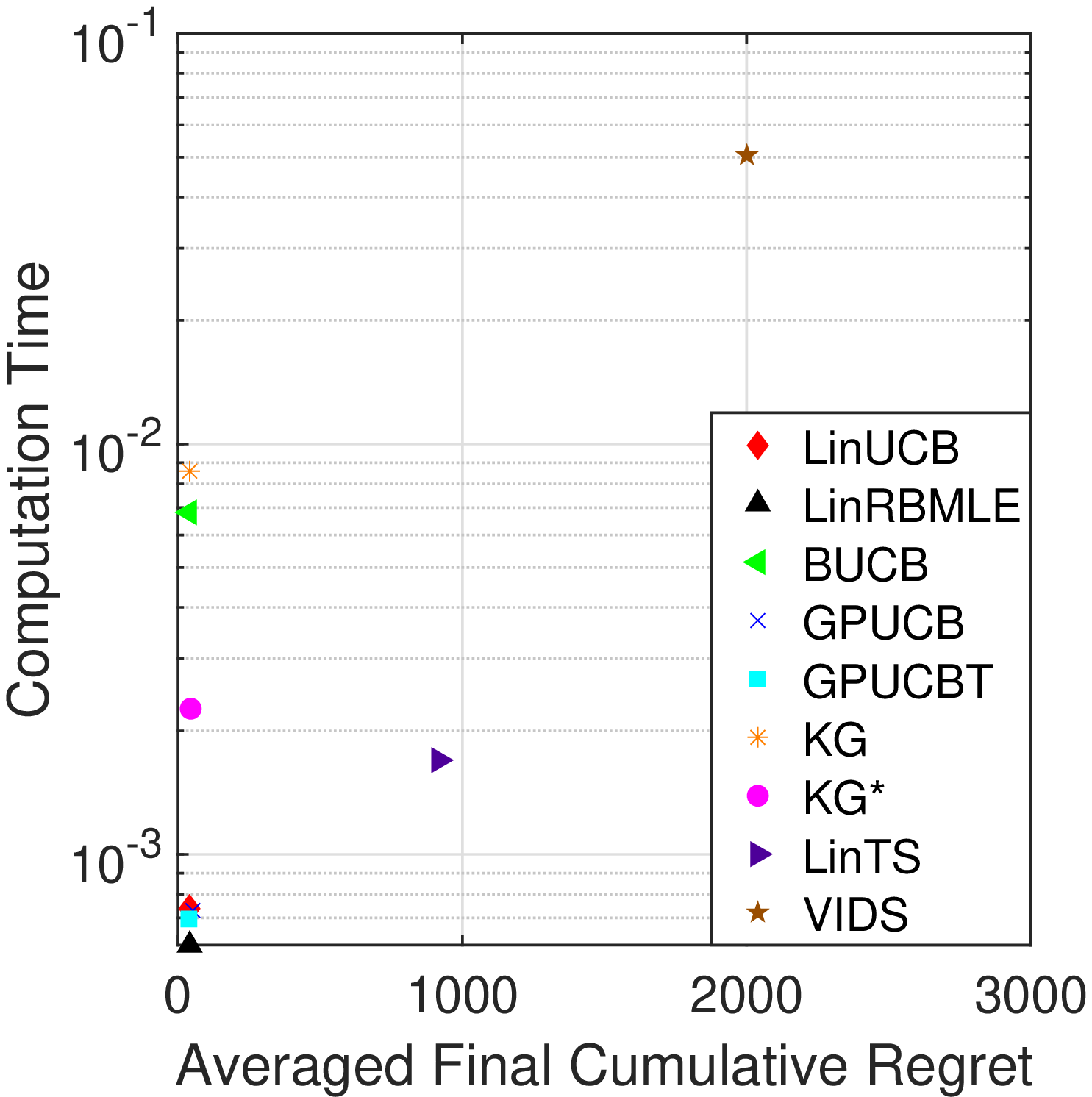}} \label{fig:time all TimeVary ID=2} &
    \hspace{-3mm}
    \scalebox{0.255}{\includegraphics[width=\textwidth]{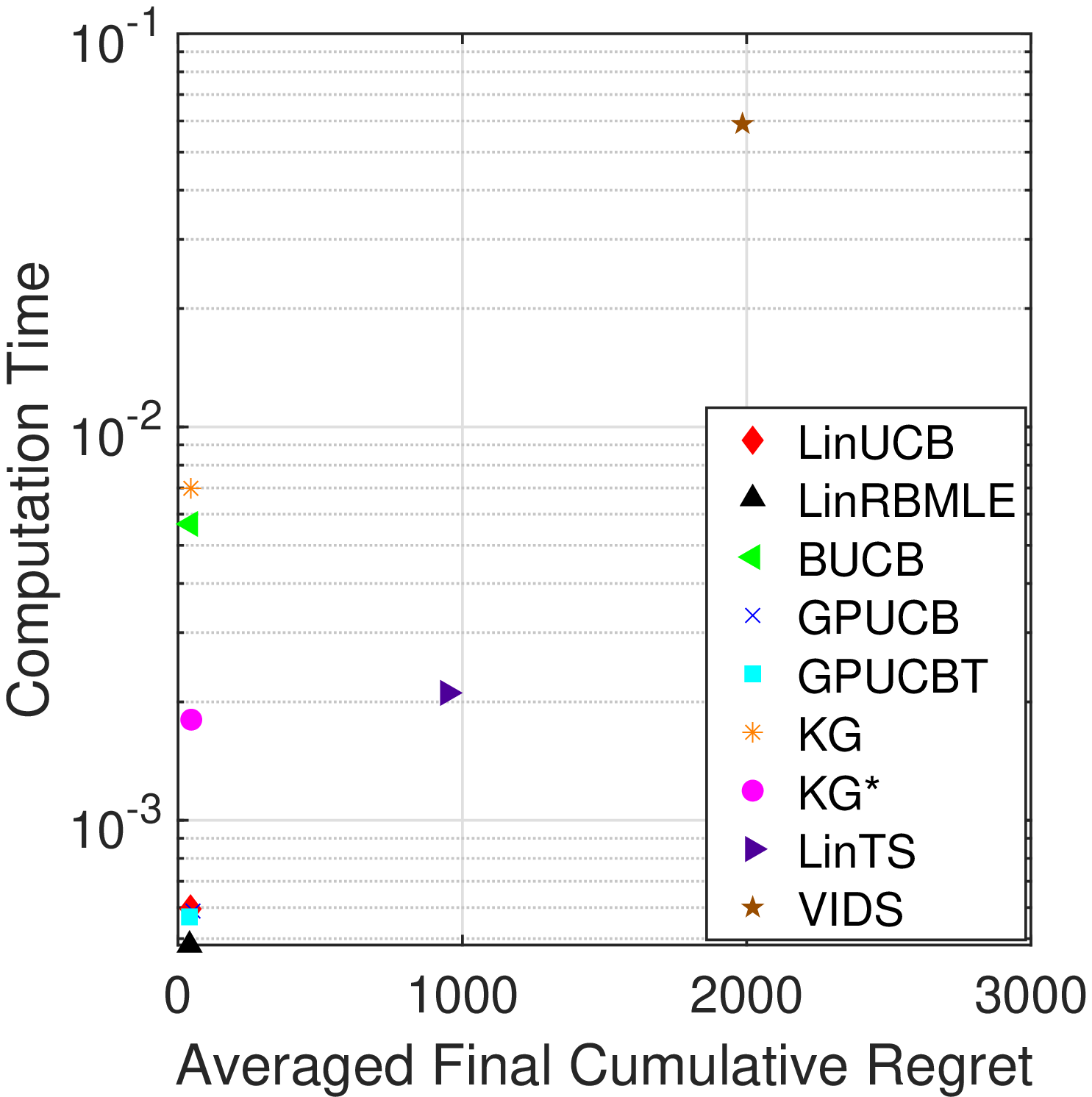}} \label{fig:time all TimeVary ID=9}\\ [0.0cm]
    \mbox{\small (a)} & \hspace{-2mm} \mbox{\small (b)} & \hspace{-2mm} \mbox{\small (c)} & \hspace{-2mm} \mbox{\small (d)}\\[-0.2cm]
\end{array}$
\caption{Average computation time per decision vs. averaged cumulative regret for (a) Figure \ref{fig:regret}(a); (b)  Figure \ref{fig:regret}(b); (c)  Figure \ref{fig:regret}(c); (d)  Figure \ref{fig:regret}(d).}
\label{fig:time all}
\vspace*{-0.1 cm}
\end{figure*}

\captionsetup[table]{labelfont={color=black}}
\begin{table*}[!h]

\footnotesize
\begin{center}
{\color{black}
\begin{tabular}{|c|c|c|c|c|c|c|c|c|c|}
\hline
\textbf{Algorithm} & \textbf{RBMLE} & \textbf{LinUCB} & \textbf{BUCB} & \textbf{GPUCB} & \textbf{GPUCBT} & \textbf{KG} & \textbf{KG*} & \textbf{LinTS} & \textbf{VIDS} \\ \hline
$d = 100, K = 100$ & {\ul\textbf{0.127}} & 0.149 & 1.157 & 0.147 & 0.145 & 1.107 & 0.401 & 0.192 & 5.054 \\ \hline
$d = 200, K = 100$ & {\ul\textbf{0.213}} & 0.24 & 1.237 & 0.234 & 0.233 & 1.168 & 0.488 & 0.561 & 9.239 \\ \hline
$d = 300, K = 100$ & {\ul\textbf{0.303}} & 0.339 & 1.467 & 0.334 & 0.332 & 1.386 & 0.599 & 1.374 & 19.876 \\ \hline
$d = 100, K = 200$ & 0.233 & 0.273 & 2.25 & 0.268 & 0.266 & 2.155 & 1.021 & {\ul {\ul\textbf{0.205}}} & 6.218 \\ \hline
$d = 200, K = 200$ & {\ul\textbf{0.373}} & 0.421 & 2.455 & 0.41 & 0.409 & 2.31 & 1.168 & 0.586 & 13.838 \\ \hline
$d = 300, K = 200$ & {\ul\textbf{0.452}} & 0.503 & 2.636 & 0.496 & 0.495 & 2.455 & 1.258 & 1.418 & 28.652 \\ \hline
\end{tabular}
}
\caption{{\color{black}Average computation time per decision for static contexts, under different values of $K$ and $d$. All numbers are averaged over 50 trials with $T= 10^2$ and in $10^{-2}$  seconds. The best is highlighted.}}
\label{table:TimeStatic_computation_time}
\end{center}
\end{table*} 
\captionsetup[table]{labelfont={color=black}}

\section{Related Work}
\label{section:related}

The RBMLE method was originally proposed in \cite{kumar1982new}.
It was subsequently examined in the Markovian setting in \cite{kumar1982optimal,kumar1983simultaneous,borkar1990kumar},
and in the linear quadratic Gaussian (LQG) system setting in \cite{kumar1983optimal,campi1998adaptive,prandini2000adaptive}.
A survey, circa 1985, of the broad field of stochastic adaptive control can be found in \cite{kumar1985survey}.
Recently it has been examined from the point of examining its regret performance in the case of non-contextual bandits with
exponential family of distributions in \cite{liu2019bandit}. Other than that, there appears to have been no work on examining its performance beyond long-term average optimality, which corresponds to regret of $o(t)$.

The linear stochastic bandits and their variants have been extensively studied from two main perspectives, namely the frequentist and the Bayesian approaches.
From the frequentist viewpoint, one major line of research is to leverage the least squares estimator and enforce exploration by constructing an upper confidence bound (UCB), introduced in the \textsc{LinReL} algorithm by \cite{auer2002using}. 
The idea of UCB was later extended to the LinUCB policy, which is simpler to implement and has been tested extensively via experiments \cite{li2010contextual}.
While being simple and empirically appealing approaches, the primitive versions of the above two algorithms are rather difficult to analyze due to the statistical dependencies among the observed rewards.
To obtain proper regret bounds, both policies were analyzed with the help of a more complicated master algorithm.
To address this issue, \cite{dani2008stochastic} proposed to construct a confidence ellipsoid, which serves as an alternative characterization of UCB, and proved that the resulting algorithm achieved an order-optimal regret bound (up to a poly-logarithmic factor).
Later, sharper characterizations of the confidence ellipsoid were presented by \cite{rusmevichientong2010linearly} and \cite{abbasi2011improved} thereby improving the regret bound.
Given the success of UCB-type algorithms for linear bandits, the idea of a confidence set was later extended to the generalized linear case \cite{filippi2010parametric,li2017provably} to study a broader class of linear stochastic bandit models.
Differing from the above UCB-type approaches, as a principled frequentist method, the RBMLE algorithm guides the exploration toward potentially reward-maximizing model parameters by applying a bias to the log-likelihood.
Most related is the work by \cite{liu2019bandit}, which adapted the RBMLE principle for stochastic multi-armed bandits and presented the regret analysis as well as extensive numerical experiments. 
However, \cite{liu2019bandit} focused on the non-contextual bandit problems, and the presented results cannot directly apply to the more structured linear bandit model.


Instead of viewing model parameters as deterministic unknown variables, the Bayesian approaches assume a prior distribution to facilitate the estimation of model parameters.
As one of the most popular Bayesian methods, Thompson sampling (TS) \cite{thompson1933likelihood} approaches the exploration issue by sampling the posterior distribution.
For linear bandit models, TS has been tested in large-scale experiments \cite{chapelle2011empirical} and shown to enjoy order-optimal regret bounds in various bandit settings \cite{agrawal2013thompson,russo2016information,abeille2017linear,agrawal2017near,dumitrascu2018pg}.
On the other hand, Bayesian strategies can also be combined with the notion of UCB for exploration, as in the popular GPUCB \cite{srinivas2010gaussian} and Bayes-UCB \cite{kaufmann2012bayesian} algorithms.
\rev{However, to the best of our knowledge, there is no regret guarantee for Bayes-UCB in the linear bandit setting \cite{urteaga2017bayesian}.}
Alternative exploration strategies for linear bandits have also been considered from the perspective of explicit information-theoretic measures.
\cite{russo2018learning} proposed a promising algorithm called information-directed sampling (IDS), which makes decisions based on the ratio between the square of expected regret and the information gain.
As the evaluation of mutual information requires computing high-dimensional integrals, VIDS, a variant of IDS, was proposed to approximate the information ratio by sampling, while still achieving competitive empirical regret performance.
Compared to IDS and its variants, the proposed RBMLE enjoys a closed-form index and is therefore computationally more efficient.
Another promising solution is the Knowledge Gradient (KG) approach \cite{ryzhov2012knowledge,ryzhov2010robustness}, which enforces exploration by taking a one-step look-ahead measurement.
While being empirically competitive, it remains unknown whether KG and its variants have a provable near-optimal regret bound.
In contrast, the proposed RBMLE enjoys provable order-optimal regret for standard linear as well as
generalized linear bandits.


\section{Conclusion}
\label{section:conclusion}
In this paper, we extend the {Reward Biased Maximum Likelihood} principle originally proposed for adaptive control, to contextual bandits.
LinRBMLE leads to a simple index policy for standard linear bandits. 
\rev{Through both theoretical regret analysis and simulations, we prove that the regret performance of LinRBMLE is competitive with the state-of-the-art methods while being computationally efficient. Given the favorable trade-off of regret and computation time, RBMLE is a promising  approach for contextual bandits.}

\section*{Ethical Impact}
Linear bandits as well as the generalized models serve as a powerful framework for sequential decision making in various critical applications, such as clinical trials \cite{varatharajah2018contextual}, mobile health \cite{tewari2017ads}, personalized recommender \cite{li2010contextual} and online advertising systems \cite{chapelle2011empirical}, etc. The rising volume of datasets in these applications requires learning algorithms that are more effective, efficient and scalable.
The study in this paper contributes a new family of frequentist approaches to this community. These approaches are proved to be order-optimal and demonstrate strong empirical performance with respect to measures of effectiveness, efficiency and scalability. As such, the proposed approaches are expected to further improve user experience in applications and benefit business stakeholders. The proposed approaches are inspired by an early adaptive control framework. This framework has been applied in many adaptive control applications \cite{kumar1985survey,kumar1982optimal,kumar1983simultaneous,kumar1983optimal,borkar1990kumar,campi1998adaptive,prandini2000adaptive}. However, analysis of its finite-time performance has been missing for decades. Our study takes a very first step towards understanding its finite-time performance in the contextual bandit setting.

Unfortunately, as in many other contextual bandit studies, our model does not take into account the fairness issue in learning the unknown parameters. For instance, it may happen that during the learning process, contextual bandit algorithms may consistently discriminate against some specific groups of users based on their social, economic, racial and sexual characteristics. Ensuring fairness may therefore require additional constraints on automated selection procedures. Such a study can contribute to general studies on the undesirable biases of machine learning algorithms \cite{joseph2016fairness}.

\bibliography{reference}

\newpage
\section*{Appendix}

\makeatletter
\newcommand*{\diff}%
  {\@ifnextchar^{\DIfF}{\DIfF^{}}}
\makeatother
\def\DIfF^#1{%
  \mathop{\mathrm{\mathstrut d}}%
  \nolimits^{#1}\gobblespace}
\def\gobblespace{%
  \futurelet\diffarg\opspace}
\def\opspace{%
  \let\DiffSpace\!%
  \ifx\diffarg(%
    \let\DiffSpace\relax
  \else
    \ifx\diffarg[%
      \let\DiffSpace\relax
    \else
      \ifx\diffarg\{%
        \let\DiffSpace\relax
      \fi%
    \fi%
  \fi%
  \DiffSpace}
\makeatother
\newcommand*{\deriv}[3][]{%
  \frac{\diff^{#1}#2}{\diff #3^{#1}}}
\newcommand*{\pderiv}[3][]{%
  \frac{\partial^{#1}#2}%
  {\partial #3^{#1}}}

\newcommand{\equad}{\mathrel{\phantom{=}}}

\appendix
\section{Proof of Indexability of the Strategy (\ref{eq:LinRBMLE index via tilde Theta t,a})}
\label{appendix:theorem:RBMLE index equivalent form}
Recall from Section \ref{section:linear:index} that $\bar{\theta}_t$ denotes a maximizer of the following problem: 
\begin{equation}
     \max_{\theta}\Big\{ \ell(\mathcal{F}_t;\theta)+\alpha(t)\cdot \max_{1 \leq a \leq K} \theta^{\intercal}x_{t,a}-\frac{\lambda}{2}{\lVert \theta \rVert}^2_2 \Big\}. \label{eq:tilde Theta t}
\end{equation}
Define
\begin{align}
    \bar{\mathcal{A}}_t&:=\argmax_{a}\hspace{2pt}\bar{\theta}_t^{\intercal}x_{t,a},\label{eq:tilde A_t}\\
    \bar{\Theta}_{t,a}&:=\argmax_{\theta} \Big\{ \ell(\mathcal{F}_t;\theta)+\alpha(t)\cdot \theta^{\intercal}x_{t,a}-\frac{\lambda}{2}{\lVert \theta \rVert}^2_2 \Big\}.\label{eq:tilde Theta t,a}
\end{align}
For each arm $a$, consider an estimator $\bar{\theta}_{t,a}\in \bar{\Theta}_{t,a}$. Subsequently, define an index set
\begin{equation}
    \bar{\mathcal{A}}'_t:=\argmax_{1\leq a\leq K} \Big\{ \ell(\mathcal{F}_t;\bar{\theta}_{t,a})+\alpha(t)\cdot \bar{\theta}_{t,a}^{\intercal}x_{t,a} -\frac{\lambda}{2}{\lVert \bar{\theta}_{t,a} \rVert}^2_2\Big\}.\label{eq:tilde A_t'}
\end{equation}
\begin{theorem}
\label{theorem:RBMLE index equivalent form}
\normalfont 
$\bar{\mathcal{A}}_t=\bar{\mathcal{A}}'_t$.
\end{theorem}
\begin{proof}
\normalfont The proof follows from the fact that any maximizer of the original double maximization problem in (\ref{eq:tilde Theta t}) remains a maximizer after interchanging the order of the max operators.
By the definition of $\bar{\Theta}_{t}$ and $\bar{\mathcal{A}}_{t}$ in (\ref{eq:tilde Theta t}) and (\ref{eq:tilde A_t}), given any $\bar{\theta}_{t}\in\bar{\Theta}_{t}$, any arm $a\in \bar{\mathcal{A}}_{t}$ is a maximizer of the optimization problem $\max_{i}\hspace{2pt}\bar{\theta}_t^{\intercal}x_{t,i}$.
We know
\begin{equation}
    \argmax_{1\leq i\leq K} \bar{\theta}_t^{\intercal}x_{t,i}=\argmax_{1\leq i\leq K}\big\{ \ell(\mathcal{F}_t;\bar{\theta}_{t})+\alpha(t)\cdot \bar{\theta}_{t}^{\intercal}x_{t,i} \big\}\label{eq:proof of index 1}.
\end{equation}
Moreover,
\begin{align}
    \max_{1\leq i\leq K} &\big\{ \ell(\mathcal{F}_t;\bar{\theta}_{t})+\alpha(t)\cdot \bar{\theta}_{t}^{\intercal}x_{t,i} \big\} \\
    &= \ell(\mathcal{F}_t;\bar{\theta}_{t}) +  \alpha(t)\cdot \max_{1\leq i\leq K}  \bar{\theta}_{t}^{\intercal}x_{t,i}\label{eq:proof of index 2}\\
    &=\max_{\theta} \big\{ \ell(\mathcal{F}_t;{\theta})+ \alpha(t)\cdot \max_{1\leq i\leq K}  {\theta}^{\intercal}x_{t,i}  \big\}\label{eq:proof of index 3}\\
    &=\max_{\theta} \max_{1\leq i\leq K}  \big\{ \ell(\mathcal{F}_t;{\theta})+ {\theta}^{\intercal}x_{t,i}    \big\}\label{eq:proof of index 4}\\
    &= \max_{1\leq i\leq K} \max_{\theta} \big\{ \ell(\mathcal{F}_t;{\theta})+ {\theta}^{\intercal}x_{t,i}    \big\}\label{eq:proof of index 5}\\
    &= \max_{1\leq i\leq K} \big\{  \ell(\mathcal{F}_t;\bar{\theta}_{t,i})+ \bar{\theta}_{t,i}^{\intercal} x_{t,i}   \big\},\label{eq:proof of index 6}
\end{align}
where (\ref{eq:proof of index 2}) follows since $\ell(\mathcal{F}_t;\bar{\theta}_{t})$ is independent of $i$, (\ref{eq:proof of index 3}) holds by the definition of $\bar{\theta}_{t}$, (\ref{eq:proof of index 4})-(\ref{eq:proof of index 5}) hold by the fact that the optimal value remains unchanged after interchanging the order of the two max operators, and (\ref{eq:proof of index 6}) follows from the definition of $\bar{\theta}_{t,i}$. Therefore, by (\ref{eq:proof of index 1})-(\ref{eq:proof of index 6}), $\bar{\mathcal{A}}_t=\bar{\mathcal{A}}'_t$.\qed
\end{proof}

\section{Proof of Corollary \ref{corollary:RBMLE index for linear bandits}}
\label{appendix:corollary:RBMLE index for linear bandits}
\normalfont By substituting the Gaussian likelihood for $\ell(\mathcal{F}_t;\theta)$ in (\ref{eq:tilde Theta t,a}), the resulting objective in (\ref{eq:tilde Theta t,a}) becomes a strictly concave function and enjoys a unique maximizer.
By the first-order necessary optimality condition \cite{bertsekas1997nonlinear}, it is easy to verify that (\ref{eq:solution for widetilde theta t,a}) is indeed the unique solution to (\ref{eq:tilde Theta t,a}).
Subsequently, based on (\ref{eq:solution for widetilde theta t,a}) and Theorem \ref{theorem:RBMLE index equivalent form}, we know the arm chosen by the RBMLE algorithm at each time $t$ is
\begin{alignat}{2}
    a_t &= \argmax_{1\leq i\leq K} \Big\{ && -(X_t\bar{\theta}_{t,i} - R_t)^{\intercal}(X_t\bar{\theta}_{t,i} - R_t) \nonumber\\
     &{} &&+ \alpha(t)\bar{\theta}_{t,i}^{\intercal}x_{i,t}-\lambda{\lVert \bar{\theta}_{t,i} \rVert}^2_2\Big\} \label{eq:corollary1:1} \\ 
    &= \argmax_{1\leq i\leq K} \Big\{ &&-\bar{\theta}_{t,i}(X^{\intercal}_tX_t+\lambda I)\bar{\theta}_{t,i} \nonumber\\
    &{} &&+ (2R^{\intercal}_t X_t + 2 \alpha(t)x_{t,i})\bar{\theta}_{t,i} -R_t^{\intercal} R_t \Big\} \label{eq:corollary1:2} \\
    &= \argmax_{1\leq i\leq K} \Big\{ &&(X_t^{\intercal} R_t + \alpha(t)x_{t,i})^{\intercal} \inv{(X_t^{\intercal}X_t + \lambda I)} \nonumber\\
    &{} && \cdot (X_t^{\intercal}R_t+\alpha(t)x_{t,i}) -R_t^{\intercal}R_t \Big\} \label{eq:corollary1:3} \\
    &= \argmax_{1\leq i\leq K} \Big\{&&\widehat{\theta}_t^{\intercal} x_{t,i} + \frac{1}{2}\alpha(t)\lVert x_{t,i} \rVert^2_{\inv{V_t}} \Big\}. \label{eq:corollary1:4} 
\end{alignat}

where (\ref{eq:corollary1:1})-(\ref{eq:corollary1:3}) hold by substituting (\ref{eq:solution for widetilde theta t,a}) in (\ref{eq:tilde A_t'}), and (\ref{eq:corollary1:4}) follows from the definition of $\widehat{\theta}_t$.\qed

\section{Proof of Lemma \ref{lemma:decomposed linear regret}}
\label{appendix:lemma:decomposed linear regret}
\begin{proof}
\normalfont
By the definition of regret for the linear bandit model,
\begin{alignat}{2}
    R_t &= \theta_*^{\intercal}x_t^* - \theta_*^{\intercal}x_t && \label{eq:lemma1:1}\\
        &= (\theta_* - \widehat{\theta}_t)^{\intercal}x_t^* + \widehat{\theta}_t^{\intercal} x_t^* 
        &&- \theta_*^{\intercal}x_t \label{eq:lemma1:2}\\
        &\leq (\theta_*- \widehat{\theta}_t)^{\intercal}x_t^* + \widehat{\theta}_t^{\intercal} x_t &&+ \frac{1}{2}\alpha(t)\lVert x_{t} \rVert^2_{V_t^{-1}} \nonumber\\
        \span &&- \frac{1}{2}\alpha(t)\lVert x^*_{t} \rVert^2_{V_t^{-1}} - \theta_*^{\intercal}x_t \label{eq:lemma1:3}\\
        &= (\theta_* - \widehat{\theta}_t)^{\intercal}x_t^* + (\widehat{\theta}_t - \theta_*)^{\intercal}x_t + \frac{1}{2}\alpha(t)\lVert x_{t} \rVert^2_{V_t^{-1}}\span \span\nonumber\\
        \span \span & - \frac{1}{2}\alpha(t)\lVert x^*_{t} \rVert^2_{V_t^{-1}},  \label{eq:lemma1:4}
\end{alignat}
where (\ref{eq:lemma1:3}) follows from the RBMLE index (\ref{eq:RBMLE index for linear bandits}).
Let $V_t^{1/2}$ and $V_t^{-1/2}$ denote square-roots, satisfying $V_t=V_t^{1/2}V_t^{1/2}$ and $\inv{V_t}=V_t^{-1/2}V_t^{-1/2}$, unique since $V_t$ is positive definite.
The result (\ref{eq:lemma1:0}) follows by replacing the vector multiplication of $(\theta_* - \widehat{\theta}_t)^{\intercal}x_t$ and $(\widehat{\theta}_t - \theta_*)^{\intercal}x_t$ in (\ref{eq:lemma1:4}) by $(\theta_* - \widehat{\theta}_t)^{\intercal}V_t^{1/2} V_t^{-1/2}x_t$ and $(\widehat{\theta}_t - \theta_*)^{\intercal}V_t^{1/2} V_t^{-1/2}x_t$, and applying the Cauchy-Schwarz inequality.\qed 
\end{proof}

\section{Proof of Theorem \ref{theorem:linear regret}}
\label{appendix:theorem:linear regret}
Before proving Theorem \ref{theorem:linear regret}, we first introduce  the following useful lemmas.
Recall that $V_t = \sum_{s=1}^{t} x_{s}x_s^{\intercal} + \lambda I$.
Moreover, recall that
\begin{align}
    G_0(t,\delta)&:={\sigma}\sqrt{{d}\log\big(\frac{\lambda+{t}}{\lambda\delta}\big)}+\lambda^{\frac{1}{2}}\\
    G_1(t)&:=\sqrt{2d\log\big(\frac{\lambda +t}{d}\big)}
\end{align}
\begin{lemma}
\label{lemma:1st lemma for proof of Theorem 2}
\normalfont
For any time $t\geq 1$, with probability at least $1-\delta$,
\begin{align}
    &\lVert\theta_*-\widehat{\theta}_t\rVert_{V_t}\cdot\lVert x_t^*\rVert_{V_t^{-1}} - \frac{1}{2}\alpha(t)\lVert x_t^*\rVert^2_{V_t^{-1}} \leq \frac{1}{2\alpha(t)}\big(G_0(t,\delta)\big)^2. \label{eq:lemma2}
\end{align}
\begin{proof}[Lemma \ref{lemma:1st lemma for proof of Theorem 2}]
\normalfont
First, we obtain an upper bound by completing the square of the left-hand side of (\ref{eq:lemma2}) as
\begin{align}
    &\lVert\theta_*-\widehat{\theta}_t\rVert_{V_t}\cdot\lVert x_t^*\rVert_{\inv{V_t}} - \frac{1}{2}\alpha(t)\lVert x_t^*\rVert^2_{\inv{V_t}} \\ 
    &= -\frac{1}{2}\alpha(t)\Big( \lVert x_t^*\rVert_{\inv{V_t}} - \frac{\lVert \theta_*-\widehat{\theta}_t \rVert_{V_t}}{\alpha(t)} \Big)^2 + \frac{1}{2}\frac{\lVert \theta_*-\widehat{\theta}_t \rVert_{V_t}^2}{\alpha(t)} \\
    &\leq \frac{1}{2}\frac{\lVert \theta_*-\widehat{\theta}_t \rVert_{V_t}^2}{\alpha(t)}.
\end{align}
Moreover, by Theorem 2 in \cite{abbasi2011improved}, we know that with probability at least $1-\delta$,
\begin{equation}
    \lVert\theta_*-\widehat{\theta}_t\rVert_{V_t}\leq {\sigma}\sqrt{{d}\log\big(\frac{\lambda+{t}}{\lambda\delta}\big)}+\lambda^{\frac{1}{2}}=G_0(t,\delta).\label{eq:Theorem 2 in Abbasi 2011}
\end{equation}
Therefore, we can conclude that (\ref{eq:lemma2}) indeed holds.\qed

\end{proof}
\end{lemma}

\begin{lemma}
\label{lemma:2nd lemma for proof of Theorem 2}
\normalfont
With probability at least $1-\delta$,
\begin{equation}
\begin{split}
    \sum_{t=1}^{T} \big(\lVert\widehat{\theta}_t - \theta_*\rVert_{V_t}\cdot \lVert x_t\rVert_{V_t^{-1}} \big) &\leq \sqrt{T}\cdot G_0(T,\delta)G_1(T)\\
    &= \mathcal{O}(\sqrt{T}\log T). \label{eq:lemma3}
\end{split}
\end{equation}
\begin{proof}[Lemma \ref{lemma:2nd lemma for proof of Theorem 2}]
\normalfont
By Lemma 11 of \cite{abbasi2011improved}, the fact that $\lVert x_{t,a}\rVert_2\leq 1$ and $\lambda\geq 1$, and the Cauchy-Schwarz inequality, we have
\begin{equation}
    \sum_{t=1}^{T} {\lVert x_t \rVert}_{\inv{V_t}} \leq \sqrt{T}\cdot G_1(T).\label{eq:upper bound on sum of context norm for linear bandits}
\end{equation}
By moving the term $\lVert\widehat{\theta}_t - \theta_*\rVert_{V_t}$ outside the summation in (\ref{eq:lemma3}) and then applying (\ref{eq:Theorem 2 in Abbasi 2011}), we obtain
\begin{equation}
    \sum_{t=1}^{T} \Big(\lVert\widehat{\theta}_t - \theta_*\rVert_{V_t}\cdot \lVert x_t\rVert_{V_t^{-1}} \Big)=\sqrt{T}\cdot G_0(T,\delta)G_1(T).
\end{equation}
This implies that (\ref{eq:lemma3}) indeed holds.\qed
\end{proof}
\end{lemma}


\begin{lemma}
\label{lemma:3rd lemma for proof of Theorem 2}
\normalfont
\begin{equation}
    \sum_{t=1}^{T} \alpha(t){\lVert x_t \rVert}_{\inv{V_t}}^2 \leq \alpha(T)\big(G_1(T)\big)^2= \mathcal{O}(\alpha(T)\log T). \label{eq:upper bound on sum of context norm square times alpha(t) for linear bandit}
\end{equation}
\begin{proof}[Lemma \ref{lemma:3rd lemma for proof of Theorem 2}]
\normalfont
by Lemma 11 of \cite{abbasi2011improved} and the fact that $\lVert x_{t,a}\rVert_2\leq 1$ and $\lambda\geq 1$, we know
\begin{equation}
        \sum_{t=1}^{T} {\lVert x_t \rVert}_{\inv{V_t}}^2 \leq \big(G_1(T)\big)^2=\mathcal{O}(\log T).\label{eq:upper bound on sum of context norm square for linear bandits}
\end{equation}
By moving the bias term outside the summation (\ref{eq:upper bound on sum of context norm square times alpha(t) for linear bandit}), we have
\begin{equation}
\begin{split}
    \sum_{t=1}^{T} \alpha(t)\lVert x_t\rVert^2_{V_t^{-1}} &\leq \alpha(T)\sum_{t=1}^{T}\lVert x_t\rVert^2_{V_t^{-1}} \\
    &\leq \alpha(T)\big(G_1(T)\big)^2 \\
    &= \mathcal{O}(\alpha(T)\log T).\label{eq:lemma 5:1}
\end{split}
\end{equation}
\qed
\end{proof}
\begin{remark}
\normalfont
Note that the first inequality in (\ref{eq:lemma 5:1}) might seem fairly conservative. However, it cannot be improved as can be seen from the following example:
Define a function $f:\mathbb{N}\rightarrow \mathbb{R}$ as: $f(t)=k + \frac{1}{t}$ if $t = 2^k$, and $f(t)=\frac{1}{t}$, otherwise.
%
%
%
It is easy to check that $\log T \leq \sum_{t=1}^{T} f(t) \leq 2\log T$, and $\sum_{t=1}^{T}\alpha(t)f(t)=\mathcal{\theta}(\alpha(T)\log T)$.
\end{remark}
\end{lemma}
Now we are ready to prove Theorem \ref{theorem:linear regret}.
\begin{proof}[Theorem \ref{theorem:linear regret}]
\normalfont
By combining (\ref{eq:lemma1:0}) and Lemmas \ref{lemma:1st lemma for proof of Theorem 2}-\ref{lemma:3rd lemma for proof of Theorem 2}, we know
\begin{equation}
\begin{split}
    \mathcal{R}(T)=\sum_{t=1}^{T}R_t &\leq \big(G_0(T,\delta)\big)^2\cdot \sum_{t=1}^{T}\frac{1}{2\alpha(t)} \\
    &+\sqrt{T}G_0(T,\delta)G_1(T)+\frac{1}{2}\alpha(T)\big(G_1(T)\big)^2.
\end{split}
\end{equation}
By choosing $\alpha(t)=\sqrt{t}$, the regret bound is
\begin{equation}
\mathcal{R}(T)=\mathcal{O}(\sqrt{T}\log T).
\end{equation}\qed
\end{proof}

\section{A Lemma for the Proof of Theorem \ref{theorem:generalized regret}}
\label{appendix:lemma:connect tilde thetas}
\begin{lemma}
\label{lemma:connect tilde thetas}
\normalfont For any arms $i$ and $j$, there exists $\bar{\theta}_0=\beta_0 \bar{\theta}_{t,i}+(1-\beta_0)\bar{\theta}_{t,j}$ with $\beta_0\in (0,1)$ such that
\begin{equation}
    (x_{t,i}+x_{t,j})^{\intercal}(\bar{\theta}_{t,j}-\bar{\theta}_{t,i})+ \alpha(t)\lVert x_{t,i}\rVert_{\inv{U_0}}-\alpha(t)\lVert x_{t,j}\rVert_{\inv{U_0}}=0,
\end{equation}
where $U_0:=\sum_{s=1}^{t-1}\mu'(x_s^{\intercal}\bar{\theta}_0)x_s x_s^{\intercal}+\lambda I$ is a $d\times d$ positive definite matrix.
\end{lemma}

\begin{proof}[Lemma \ref{lemma:connect tilde thetas}]
\normalfont
By (\ref{eq:tilde theta first-order condition}),
\begin{align}
    &\sum_{s=1}^{t-1} \big(r_s x_s - \mu(x_s^{\intercal}\bar{\theta}_{t,i})x_s\big) - \lambda \bar{\theta}_{t,i}+ \alpha(t)x_{t,i}=0,\label{eq:tilde theta first-order condition of arm i}\\
    &\sum_{s=1}^{t-1} \big(r_s x_s - \mu(x_s^{\intercal}\bar{\theta}_{t,j})x_s\big) - \lambda \bar{\theta}_{t,j}+ \alpha(t)x_{t,j}=0.\label{eq:tilde theta first-order condition of arm j}
\end{align}
Moreover, by the mean value theorem, there exists $\beta_0\in(0,1)$ and $\underbar{$\theta$}=\beta_0 \bar{\theta}_{t,i}+(1-\beta_0)\bar{\theta}_{t,j}$ such that
\begin{align}
    &\sum_{s=1}^{t-1}\mu(x_s^{\intercal}\bar{\theta}_{t,i})x_s +\lambda \bar{\theta}_{t,i}- \sum_{s=1}^{t-1}\mu(x_s^{\intercal}\bar{\theta}_{t,j})x_s -\lambda \bar{\theta}_{t,i}\\
    &=\Big[\sum_{s=1}^{t} \mu'(x_s^{\intercal}\underbar{$\theta$})x_s x_s^{\intercal} +\lambda I\Big](\bar{\theta}_{t,i}-\bar{\theta}_{t,j})=U_0(\bar{\theta}_{t,i}-\bar{\theta}_{t,j}).
\end{align}
Multiplying both sides of (\ref{eq:tilde theta first-order condition of arm i})-(\ref{eq:tilde theta first-order condition of arm j}) by the row vector $(x_{t,i}+x_{t,j})^{\intercal}\inv{U_0}$ yields
\begin{alignat}{2}
    & (x_{t,i}+x_{t,j})^{\intercal}  \inv{U_0}\Big(&&\sum_{s=1}^{t-1}\big(r_s x_s - \mu(x_s^{\intercal}\bar{\theta}_{t,i})x_s\big)-\lambda \bar{\theta}_{t,i}\Big)  \nonumber\\
    \span \span &+ \alpha(t)(x_{t,i}+x_{t,j})^{\intercal}\inv{U_0} x_{t,i}=0,\label{eq:tilde theta first-order condition of arm i after multiplication}\\
    & (x_{t,i}+x_{t,j})^{\intercal}  \inv{U_0}\Big(&&\sum_{s=1}^{t-1}\big(r_s x_s - \mu(x_s^{\intercal}\bar{\theta}_{t,j})x_s\big)-\lambda \bar{\theta}_{t,j}\Big) \nonumber\\
    \span \span &+ \alpha(t) (x_{t,i}+x_{t,j})^{\intercal}\inv{U_0}x_{t,j}=0.\label{eq:tilde theta first-order condition of arm j after multiplication}
\end{alignat}
By combining (\ref{eq:tilde theta first-order condition of arm i after multiplication})-(\ref{eq:tilde theta first-order condition of arm j after multiplication}) and eliminating the common terms, we conclude that
\begin{equation}
    (x_{t,i}+x_{t,j})^{\intercal}(\bar{\theta}_{t,j}-\bar{\theta}_{t,i})+ \alpha(t)\lVert x_{t,i}\rVert_{\inv{U_0}}-\alpha(t)\lVert x_{t,j}\rVert_{\inv{U_0}}=0.
\end{equation}
\qed
\end{proof}

\section{Proof of Theorem \ref{theorem:generalized regret}}
\label{appendix:theorem:generalized regret}
For each time $t$, we denote the estimate of $\theta$ without applying the bias term as $\widehat{\theta}_t$, which satisfies the first-order necessary condition $\nabla_{\theta}(\ell(\mathcal{F}_t;\theta)-\frac{\lambda}{2}\lVert \theta\rVert_2^2)\rvert_{\theta=\widehat{\theta}_t}=0$.
Equivalently,
\begin{equation}
    \sum_{s=1}^{t-1} \big(r_s x_s - \mu(x_s^{\intercal}\widehat{\theta}_{t})x_s\big)-\lambda \widehat{\theta}_t=0.\label{eq:hat theta first-order condition}
\end{equation}
Recall that $V_t=\sum_{s=1}^{t-1}x_s x_s^{\intercal}+\lambda I$, where $I$ denotes the $d\times d$ identity matrix.
{Without loss of generality, we may assume that $L_{\mu}\geq 1$ and $\kappa_{\mu}\leq 1$ (as these can be easily achieved by adding a constant scaling factor to the link function).}
Before proving Theorem \ref{theorem:generalized regret}, we first establish several preliminary results.
\begin{lemma}
\label{lemma:norm of theta hat minus tilde theta}
\normalfont For any arm $i$,
\begin{equation}
    \lVert \widehat{\theta}_{t}-\bar{\theta}_{t,i}\rVert_{V_t} \leq \frac{1}{\kappa_{\mu}}\alpha(t)\lVert x_{t,i}\rVert_{\inv{V_t}}.\label{eq:norm of hat theta - tilde theta}
\end{equation}
\end{lemma}
\begin{proof}[Lemma \ref{lemma:norm of theta hat minus tilde theta}]
\normalfont 
For each time $t$, define a ``helper function" $Z_t(\cdot):\mathbb{R}^{d}\rightarrow \mathbb{R}^d$ by
\begin{equation}
    Z_{t}(\theta):=\sum_{s=1}^{t-1} \big(\mu(x_s^{\intercal}\theta)-\mu(x_s^{\intercal}\theta_*)\big)x_s+\lambda(\theta-\theta_*).\label{eq:Z_t}
\end{equation}
It is easy to verify that $Z_{t}(\theta_*)=0$. By (\ref{eq:tilde theta first-order condition}),
\begin{equation}
\begin{split}
    Z_t(\widehat{\theta}_t)-Z_t(\bar{\theta}_{t,i})&=\sum_{s=1}^{t-1}\Big( \big(\mu(x_s^{\intercal}\widehat{\theta}_t)-\mu(x_s^{\intercal}\bar{\theta}_{t,i})\big)x_s\Big) \\
    &+\lambda(\widehat{\theta}_t-\bar{\theta}_{t,i})=-\alpha(t)x_{t,i}.\label{eq:Z_t difference}
\end{split}
\end{equation}
Next, we consider upper and lower bounds on the inner product of $\widehat{\theta}_t-\bar{\theta}_{t,i}$ and $Z_t(\widehat{\theta}_t)-Z_t(\bar{\theta}_{t,i})$. For the upper bound,
\begin{align}
   (\widehat{\theta}_t-\bar{\theta}_{t,i})^{\intercal}(Z_t(\widehat{\theta}_t)&-Z_t(\bar{\theta}_{t,i}))\nonumber\\
   &=-\alpha(t)(\widehat{\theta}_t-\bar{\theta}_{t,i})^{\intercal}x_{t,i}\label{eq:hat theta - tilde theta 1-1}\\
   &\leq  \alpha(t)\lVert \widehat{\theta}_t-\bar{\theta}_{t,i}\Vert_{V_t}\cdot \lVert x_{t,i} \rVert_{\inv{V_t}},\label{eq:hat theta - tilde theta 1-2}
\end{align}
where (\ref{eq:hat theta - tilde theta 1-1}) follows from (\ref{eq:Z_t difference}),  and (\ref{eq:hat theta - tilde theta 1-2}) holds by the Cauchy-Schwarz inequality.
Similarly, we obtain a lower bound as
\begin{align}
    (\widehat{\theta}_t-\bar{\theta}_{t,i})^{\intercal}(Z_t(\widehat{\theta}_t)-Z_t(\bar{\theta}_{t,i}))&\geq (\widehat{\theta}_t-\bar{\theta}_{t,i})^{\intercal} \kappa_{\mu}V_{t}(\widehat{\theta}_t-\bar{\theta}_{t,i})\label{eq:hat theta - tilde theta 2-1}\\
    &=\kappa_{\mu}\lVert \widehat{\theta}_t-\bar{\theta}_{t,i}\rVert_{V_t}^2.\label{eq:hat theta - tilde theta 2-2}
\end{align}
By combining (\ref{eq:hat theta - tilde theta 1-2}) and (\ref{eq:hat theta - tilde theta 2-2}), we conclude that (\ref{eq:norm of hat theta - tilde theta}) indeed holds. \qed
\end{proof}

Based on Lemma \ref{lemma:connect tilde thetas}, given $\bar{\theta}_{t}$ and $\bar{\theta}_{t,a_t^*}$, there must exist a constant $\beta_0\in (0,1)$, satisfying $\underbar{$\theta$}=\beta_0 \bar{\theta}_{t}+(1-\beta_0)\bar{\theta}_{t,a_t^*}$, such that
\begin{equation}
    (x_{t}+x_{t,a_t^*})^{\intercal}(\bar{\theta}_{t,a_t^*}-\bar{\theta}_{t})+ \alpha(t)\lVert x_{t}\rVert_{\inv{U}}-\alpha(t)\lVert x_{t,a_t^*}\rVert_{\inv{U}}=0,\label{eq:generalized regret 0}
\end{equation}
where the matrix $U$ is defined as
\begin{equation}
     U=\sum_{s=1}^{t-1}\mu'(x_s^{\intercal}\underbar{$\theta$})x_s x_s^{\intercal}+\lambda I.\label{eq: U definition} 
\end{equation}
For ease of notation, we define the $L_2$-regularized log-likelihood as
\begin{equation}
     \ell_{\lambda}(\mathcal{F}_t;{\theta}):=\ell(\mathcal{F}_t;{\theta})-\frac{\lambda}{2}\lVert \theta\rVert_2^2
\end{equation}
\begin{lemma}
\label{lemma:log-likelihood difference}
\normalfont For any arm $i$, the $L_2$-regularized log-likelihood satisfies
\begin{equation}
    \ell_{\lambda}(\mathcal{F}_t;\bar{\theta}_t)-\ell_{\lambda}(\mathcal{F}_t;\bar{\theta}_{t,i})\leq \frac{L_{\mu}}{2 \kappa_{\mu}^2}\cdot\alpha(t)^2\lVert x_{t,i} \rVert_{\inv{V_t}}^2.\label{eq:log-likelihood difference}
\end{equation}
\end{lemma}
\begin{proof}[Lemma \ref{lemma:log-likelihood difference}]
\normalfont We quantify the difference in log-likelihood under $\bar{\theta}_t$ and $\bar{\theta}_{t,i}$ with the help of $\widehat{\theta}_t$. Denoting the Hessian of $\ell_{\lambda}(\mathcal{F}_t;{\theta})$ with respect to $\theta$ by $H_{\ell}(\theta)$, we have
\begin{equation}
    H_{\ell}(\theta)=\sum_{s=1}^{t-1}-\mu'(x_s^{\intercal}\theta)x_s x_s^{\intercal}-\lambda I,\label{eq:Hessian of ell}
\end{equation}
and hence $H_{\ell}(\theta)$ is negative-definite.
By the boundedness of $\mu'$, we also know that
\begin{equation}
    H_{\ell}(\theta)\succeq -L_{\mu}(V_t-\lambda I)-\lambda I\succeq -L_{\mu}V_t.\label{eq:H is larger than -L_mu V_t}
\end{equation}
Consequently,
\begin{alignat}{2}
    & \ell_{\lambda}(\mathcal{F}_t;\bar{\theta}_t)-\ell_{\lambda}(\mathcal{F}_t;\bar{\theta}_{t,i})=\span\span \nonumber\\ 
    \big(&  \ell_{\lambda}(\mathcal{F}_t;\bar{\theta}_t)-\ell_{\lambda}(\mathcal{F}_t;\widehat{\theta}_t)\big) \span\span+ \big(\ell_{\lambda}(\mathcal{F}_t;\widehat{\theta}_t) - \ell_{\lambda}(\mathcal{F}_t;\bar{\theta}_{t,i})\big) \nonumber \\
    & = &&\underbrace{\frac{1}{2}(\bar{\theta}_t-\widehat{\theta}_t)^{\intercal}H_{\ell}(\theta')(\bar{\theta}_t-\widehat{\theta}_t)}_{\leq 0} \nonumber\\
     &{}&&-\frac{1}{2}(\bar{\theta}_{t,i}-\widehat{\theta}_t)^{\intercal}H_{\ell}(\theta'')(\bar{\theta}_{t,i}-\widehat{\theta}_t)\label{eq:log-likelihood difference 1-1}\\
    &\leq \frac{1}{2}L_{\mu}\cdot \lVert \bar{\theta}_{t,i}-\widehat{\theta}_t\rVert_{V_t}^{2}\label{eq:log-likelihood difference 1-2}\span \span\\
    &\leq \frac{L_{\mu}}{2 \kappa_{\mu}^2}\cdot\alpha(t)^2\lVert x_{t,i} \rVert_{\inv{V_t}}^2,\span \span\label{eq:log-likelihood difference 1-3}
\end{alignat}
where (\ref{eq:log-likelihood difference 1-1}) follows from (\ref{eq:hat theta first-order condition}) and the Taylor expansion of $\ell_{\lambda}(\mathcal{F}_t;\theta)$ at $\theta=\widehat{\theta}_t$ up to the quadratic term (with $\theta'=\xi' \bar{\theta}_t+(1-\xi')\widehat{\theta}_t$ and $\theta''=\xi'' \bar{\theta}_{t,i}+(1-\xi'')\widehat{\theta}_t$ for some $\xi',\xi''\in[0,1]$), 
(\ref{eq:log-likelihood difference 1-2}) holds by (\ref{eq:H is larger than -L_mu V_t}), and (\ref{eq:log-likelihood difference 1-3}) is a direct result of Lemma \ref{lemma:norm of theta hat minus tilde theta}.\qed
\end{proof}

As will be seen presently, the regret bound involves several quantities concerning the norms of the differences in the estimators of $\theta$ and the norms of the context vectors. Recalling that $\widehat{\theta}_t$ denotes the estimator of $\theta$ without applying the bias term, we first establish several useful inequalities in the following Lemma \ref{lemma:useful inequalities}.
For ease of exposition, we discuss the Loewner order of the two key matrices $V_t$ and $U$. For any two symmetric matrices $A,B$, we write $A\preceq B$ if $B-A$ is a positive semi-definite matrix.
Similarly, we write $A\succeq B$ if $A-B$ is positive semi-definite.
By (\ref{eq: U definition}), the boundedness of the first-order derivative of $\mu$ and that $L_{\mu}\geq 1$, we know 
\begin{equation}
    U\preceq L_{\mu}(V_t-\lambda I)+\lambda I = L_{\mu}V_t+(\lambda-L_{\mu}\lambda)I\preceq L_{\mu}V_t.
\end{equation}
Similarly, by the fact that $\kappa_{\mu}\leq 1$, we have
\begin{equation}
     U\succeq \kappa_{\mu}(V_t-\lambda I)+\lambda I=\kappa_{\mu}V_t+(\lambda-\kappa_{\mu}\lambda)I\succeq \kappa_{\mu}V_t.
\end{equation}
\begin{lemma}
\label{lemma:useful inequalities}
\normalfont The following inequalities hold with probability one:
\begin{align}
     {\lVert \widehat{\theta}_{t}-\bar{\theta}_{t}\rVert}_{U}\cdot \lVert x_{t,a_t^*} \rVert_{\inv{U}} &\leq \frac{L_{\mu}^2}{\kappa_{\mu}}\alpha(t)\lVert x_t\rVert_{\inv{U}} \cdot\lVert x_{t,a_t^*}\rVert_{\inv{U}}\label{eq:useful inequalities 1-1},\\
     {\lVert {\theta}_*-\widehat{\theta}_{t}\rVert}_{U}\cdot \lVert x_{t,a_t^*} \rVert_{\inv{U}} &\leq
     {L_{\mu}\lVert \theta_*-\widehat{\theta}_t\rVert_{V_t}}\cdot \lVert x_{t,a_t^*}\rVert_{\inv{U}}\label{eq:useful inequalities 1-2},\\
     {\lVert {\theta}_*-\widehat{\theta}_{t}\rVert}_{U} \cdot \lVert x_t\rVert_{\inv{U}}&\leq \frac{L_{\mu}}{\kappa_{\mu}}{\lVert {\theta}_*-\widehat{\theta}_{t}\rVert}_{V_t}\cdot \lVert x_t\rVert_{\inv{V_t}}\label{eq:useful inequalities 1-3},\\
     {\lVert \widehat{\theta}_{t}-\bar{\theta}_{t,a_t^*}\rVert}_{U}\cdot \lVert x_{t} \rVert_{\inv{U}}&\leq \frac{L_{\mu}^2}{\kappa_{\mu}}\alpha(t)\lVert x_t\rVert_{\inv{U}} \cdot\lVert x_{t,a_t^*}\rVert_{\inv{U}}.\label{eq:useful inequalities 1-4}
\end{align}
\end{lemma}
\begin{proof}[Lemma \ref{lemma:useful inequalities}]
\normalfont For (\ref{eq:useful inequalities 1-1}), it can be shown that
\begin{align}
    {\lVert \widehat{\theta}_{t}-\bar{\theta}_{t}\rVert}_{U}&\cdot \lVert x_{t,a_t^*} \rVert_{\inv{U}}\nonumber\\
    &\leq L_{\mu} {\lVert \widehat{\theta}_{t}-\bar{\theta}_{t}\rVert}_{V_t}\cdot\lVert x_{t,a_t^*} \rVert_{\inv{U}} \label{eq:useful inequalities 2-1}\\
    &\leq L_{\mu} \Big( \frac{1}{\kappa_{\mu}}\alpha(t){\lVert x_{t,a_t^*} \rVert_{\inv{V_t}}}\Big) {\lVert x_{t,a_t^*} \rVert}_{\inv{U}}\label{eq:useful inequalities 2-2}\\
    &\leq \frac{L_{\mu}^2}{\kappa_{\mu}}\alpha(t)\lVert x_t\rVert_{\inv{U}} \cdot\lVert x_{t,a_t^*}\rVert_{\inv{U}},\label{eq:useful inequalities 2-3}
\end{align}
where (\ref{eq:useful inequalities 2-1}) and (\ref{eq:useful inequalities 2-3}) hold by the definition of $U$ in (\ref{eq: U definition}) and the boundedness of the first-order derivative of $\mu$, and (\ref{eq:useful inequalities 2-2}) is a direct result of Lemma \ref{lemma:norm of theta hat minus tilde theta}.
Similarly, (\ref{eq:useful inequalities 1-4}) can be shown by following the same procedure as (\ref{eq:useful inequalities 2-1})-(\ref{eq:useful inequalities 2-3}).
For (\ref{eq:useful inequalities 1-2}) and (\ref{eq:useful inequalities 1-3}), by the definition of $U$ and the boundedness of the first-order derivative of $\mu$, it is easy to verify that (\ref{eq:useful inequalities 1-2}) and (\ref{eq:useful inequalities 1-3}) indeed hold. \qed
\end{proof}
Now we are ready to prove Theorem \ref{theorem:generalized regret}.
\begin{proof}[Theorem \ref{theorem:generalized regret}]
\normalfont
To begin with, recall from Section \ref{section:generalized:index} that at each time $t$, GLM-RBMLE selects an arm from the index set $\bar{\mathcal{A}}''_t$ defined as
\begin{equation}
     \bar{\mathcal{A}}''_t:=\argmax_{1\leq a\leq K} \Big\{ \ell(\mathcal{F}_t;\bar{\theta}_{t,a})+\eta(t)\alpha(t)\cdot \bar{\theta}_{t,a}^{\intercal}x_{t,a}-\frac{\lambda}{2}{\lVert \bar{\theta}_{t,a}\rVert}_2^2  \Big\},\label{eq:tilde A_t''}
\end{equation}
Recall that the \textit{immediate regret} is defined as $R_t=\mu(\theta_{*}^{\intercal}x_{t,a_t^*})-\mu(\theta_{*}^{\intercal}x_{t})$.
By (\ref{eq:generalized regret 0}), under the GLM-RBMLE index in (\ref{eq:tilde A_t''}),
\begin{alignat}{2}
    0 &\geq \Big(\bar{\theta}_{t,a_t^*}^{\intercal}x_{t,a_t^*}+\frac{\ell_{\lambda}(\mathcal{F}_t;\bar{\theta}_{t,a_t^*})}{\eta(t)\alpha(t)}\Big) - \Big(\bar{\theta}_{t}^{\intercal}x_{t}+\frac{\ell_{\lambda}(\mathcal{F}_t;\bar{\theta}_{t})}{\eta(t)\alpha(t)}\Big)\label{eq:generalized regret 0-1}\\
    &=\bar{\theta}_{t}^{\intercal} x_{t,a_t^*}-\bar{\theta}_{t,a_t^*}^{\intercal} x_{t}-\alpha(t)\lVert x_{t} \rVert_{\inv{U}} + \alpha(t)\lVert x_{t,a_t^*} \rVert_{\inv{U}} \nonumber\\ 
    -& \frac{\ell_{\lambda}(\mathcal{F}_t;\bar{\theta}_t)}{\eta(t)\alpha(t)}+\frac{\ell_{\lambda}(\mathcal{F}_t;\bar{\theta}_{t,a_t^*})}{\eta(t)\alpha(t)}.\label{eq:generalized regret 0-2}
\end{alignat}
Hence,
\begin{align}
    R_t \leq & L_{\mu}\cdot(\theta_{*}^{\intercal}x_{t,a_t^*}-\theta_{*}^{\intercal}x_{t})\label{eq:generalized regret 1-1}\\
    =&L_{\mu}\cdot\big[(\theta_{*}-\bar{\theta}_{t})^{\intercal}x_{t,a_t^*}-\theta_{*}^{\intercal}x_{t}+\bar{\theta}_{t}^{\intercal}x_{t,a_t^*}\big]\label{eq:generalized regret 1-2}\\
    \leq &L_{\mu}\cdot \bigg[(\theta_{*}-\bar{\theta}_{t})^{\intercal}x_{t,a_t^*}- \theta_{*}^{\intercal}x_{t} \nonumber\\
    &+ \Big(x_t^{\intercal}\bar{\theta}_{t,a_{t}^*}+\alpha(t)\lVert x_{t} \rVert_{\inv{U}} - \alpha(t)\lVert x_{t,a_t^*} \rVert_{\inv{U}} \nonumber\\
    &+ \frac{\ell_{\lambda}(\mathcal{F}_t;\bar{\theta}_t)}{\eta(t)\alpha(t)} -\frac{\ell_{\lambda}(\mathcal{F}_t;\bar{\theta}_{t,a_t^*})}{\eta(t)\alpha(t)} \Big)\bigg]\label{eq:generalized regret 1-3}\\
    = & L_{\mu}\cdot \bigg[ (\theta_{*}-\bar{\theta}_{t})^{\intercal}x_{t,a_t^*}+(\bar{\theta}_{t,a_t^*}-\theta_{*})^{\intercal}x_{t} \nonumber\\
    &+\alpha(t)\lVert x_{t} \rVert_{\inv{U}} - \alpha(t)\lVert x_{t,a_t^*} \rVert_{\inv{U}} \nonumber\\
    &+ \frac{\ell_{\lambda}(\mathcal{F}_t;\bar{\theta}_t)}{\eta(t)\alpha(t)} -\frac{\ell_{\lambda}(\mathcal{F}_t;\bar{\theta}_{t,a_t^*})}{\eta(t)\alpha(t)}\bigg]\label{eq:generalized regret 1-4}\\
    \leq & L_{\mu}\cdot \bigg[\lVert \theta_{*}-\bar{\theta}_{t}\rVert_{U} \cdot\lVert x_{t,a_t^*}\rVert_{\inv{U}}\nonumber\\
    &+ \lVert \bar{\theta}_{t,a_t^*}-\theta_{*}\rVert_{U}\cdot \lVert x_t\rVert_{\inv{U}}\nonumber \\
    &+\alpha(t)\lVert x_{t} \rVert_{\inv{U}} - \alpha(t)\lVert x_{t,a_t^*} \rVert_{\inv{U}} \nonumber\\
    &+ \frac{\ell_{\lambda}(\mathcal{F}_t;\bar{\theta}_t)}{\eta(t)\alpha(t)} -\frac{\ell_{\lambda}(\mathcal{F}_t;\bar{\theta}_{t,a_t^*})}{\eta(t)\alpha(t)}\bigg],\label{eq:generalized regret 1-5}
\end{align}
where (\ref{eq:generalized regret 1-1}) follows from the the boundedness of the derivative of $\mu$, (\ref{eq:generalized regret 1-3})-(\ref{eq:generalized regret 1-4}) hold by (\ref{eq:generalized regret 0-1})-(\ref{eq:generalized regret 0-2}), and (\ref{eq:generalized regret 1-5}) is a direct result of the Cauchy-Schwarz inequality with respect to the norm induced by the matrix $U$.
Next, we provide an upper bound for each term in (\ref{eq:generalized regret 1-5}):
\begin{itemize}[leftmargin=*]
    \item $ \lVert\theta_{*}-\bar{\theta}_{t}\rVert_{U} \cdot\lVert x_{t,a_t^*}\rVert_{\inv{U}}$: We can obtain an upper bound by applying the Cauchy-Schwarz inequality and (\ref{eq:useful inequalities 1-1})-(\ref{eq:useful inequalities 1-2}) in Lemma \ref{lemma:useful inequalities}, as
    \begin{align}
        \lVert\theta_{*}-\bar{\theta}_{t}\rVert_{U} &\cdot\lVert x_{t,a_t^*}\rVert_{\inv{U}} \nonumber \\
        \leq & \big(\lVert\theta_{*}-\widehat{\theta}_t\rVert_{U}+\lVert \widehat{\theta}_t-\bar{\theta}_{t}\rVert_{U}\big)\lVert x_{t,a_t^*}\rVert_{\inv{U}}\\
        \leq & {L_{\mu}\lVert \theta_*-\widehat{\theta}_t\rVert_{V_t}}\cdot \lVert x_{t,a_t^*}\rVert_{\inv{U}} \nonumber \\
        &+ \frac{L_{\mu}^2}{\kappa_{\mu}}\alpha(t)\lVert x_t\rVert_{\inv{U}} \cdot\lVert x_{t,a_t^*}\rVert_{\inv{U}}.
    \end{align}
    \item $\lVert \bar{\theta}_{t,a_t^*}-\theta_{*}\rVert_{U}\cdot \lVert x_t\rVert_{\inv{U}}$: By the Cauchy-Schwarz inequality and (\ref{eq:useful inequalities 1-4}) in Lemma \ref{lemma:useful inequalities},
    \begin{align}
        \lVert \bar{\theta}_{t,a_t^*}-\theta_{*}\rVert_{U} & \cdot \lVert x_t\rVert_{\inv{U}} \nonumber\\
        \leq & \Big(\lVert \bar{\theta}_{t,a_t^*}-\widehat{\theta}_{t}\rVert_{U} + \lVert \widehat{\theta}_{t}-{\theta}_{*}\rVert_{U}\Big) \lVert x_t\rVert_{\inv{U}}\\ 
        \leq & \frac{L_{\mu}^2}{\kappa_{\mu}}\alpha(t)\lVert x_t\rVert_{\inv{U}} \cdot\lVert x_{t,a_t^*}\rVert_{\inv{U}} \nonumber \\
        &+ \frac{L_{\mu}}{\kappa_{\mu}}\lVert \widehat{\theta}_{t}-{\theta}_{*}\rVert_{U} \cdot {\lVert x_t \rVert }_{\inv{V_t}}.
    \end{align}
    \item $\alpha(t)\lVert x_{t} \rVert_{\inv{U}}$: It is easy to verify that
    \begin{equation}
        \alpha(t)\lVert x_{t} \rVert_{\inv{U}}\leq \frac{1}{\kappa_{\mu}^2}\alpha(t)\lVert x_t\rVert_{\inv{V_t}}.
    \end{equation}
    \item $\frac{\ell_{\lambda}(\mathcal{F}_t;\bar{\theta}_t)}{\eta(t)\alpha(t)} -\frac{\ell_{\lambda}(\mathcal{F}_t;\bar{\theta}_{t,a_t^*})}{\eta(t)\alpha(t)}$: By Lemma \ref{lemma:log-likelihood difference}, we know
    \begin{equation}
    \begin{split}
        \frac{\ell_{\lambda}(\mathcal{F}_t;\bar{\theta}_t)}{\eta(t)\alpha(t)} -\frac{\ell_{\lambda}(\mathcal{F}_t;\bar{\theta}_{t,a_t^*})}{\eta(t)\alpha(t)}&\leq \frac{L_{\mu}}{2\eta(t) \kappa_{\mu}^2}\cdot\alpha(t)\lVert x_{t,a_t^*} \rVert_{\inv{V_t}}^2 \\
        &\leq \frac{L_{\mu}^3}{2\eta(t) \kappa_{\mu}^2}\cdot\alpha(t)\lVert x_{t,a_t^*} \rVert_{\inv{U}}^2.
    \end{split}
    \end{equation}
\end{itemize}
By combining (\ref{eq:generalized regret 1-5}) and the above upper bounds, we have
\begin{align}
    R_t \leq & L_{\mu}\bigg[\Big( \big(\frac{L_{\mu}^3}{2\kappa_{\mu}^2 \eta(t)}-1\big)\alpha(t)  \Big)\cdot {\lVert x_{t,a_t^*}\rVert}_{\inv{U}}^2  
    + \nonumber\\
    &\Big(\frac{2 L_{\mu}^2}{\kappa_{\mu}}\alpha(t){\lVert x_t\rVert}_{\inv{U}}+L_{\mu}{\lVert \theta_{*}-\widehat{\theta}_t\rVert}_{V_t}  \Big){\lVert x_{t,a_t^*}\rVert}_{\inv{U}}\label{eq:generalized regret 2-1}\\
    &+\Big(\frac{L_{\mu}}{\kappa_{\mu}}{\lVert \widehat{\theta}_t - \theta_{*}\rVert}_{V_t}\cdot {\lVert x_t\rVert}_{\inv{V_t}} + \frac{1}{\kappa_{\mu}^2} \alpha(t) {\lVert x_t\rVert}_{\inv{V_t}}    \Big)
    \bigg]. \label{eq:generalized regret 2-2}
\end{align}
Note that (\ref{eq:generalized regret 2-1})-(\ref{eq:generalized regret 2-2}) can be interpreted as a quadratic function of ${\lVert x_{t,a_t^*}\rVert}_{\inv{U}}$. 
Recall that $T_0:=\min\{t\in\mathbb{N}: \frac{L_{\mu}^3}{2\kappa_{\mu}^2\eta(t)}<1\}$.
Therefore, for any $t\geq T_0$, by completing the square,
\begin{align}
    R_t \leq & L_{\mu}\bigg[\frac{\alpha(t)}{4 (1-\frac{L_{\mu}^3}{2\kappa_{\mu}^2 \eta(t)})}  \Big(  \frac{2 L_{\mu}^2}{\kappa_{\mu}}{\lVert x_t\rVert}_{\inv{U}} + L_{\mu}\frac{{\lVert \theta_{*}-\widehat{\theta}_t\rVert}_{V_t}}{\alpha(t)}\Big)^2 \label{eq:generalized regret 3-1}\\
    &+\frac{L_{\mu}}{\kappa_{\mu}}{\lVert \widehat{\theta}_t - \theta_{*}\rVert}_{V_t}\cdot {\lVert x_t\rVert}_{\inv{V_t}} + \frac{1}{\kappa_{\mu}^2} \alpha(t) {\lVert x_t\rVert}_{\inv{V_t}}^2   \bigg]. \label{eq:generalized regret 3-2}
\end{align}
Based on  (\ref{eq:generalized regret 3-1})-(\ref{eq:generalized regret 3-2}), to bound the cumulative regret, we need the following properties. Recall that $G_1(t)$ and $G_2(t,\delta)$ are defined as
\begin{align}
    G_1(t)&:=\sqrt{2d\log\big(\frac{\lambda +t}{d}\big)}\\
    G_2(t,\delta)&:=\frac{\sigma}{\kappa_{\mu}}\sqrt{\frac{d}{2}\log(1+\frac{2t}{d})+\log\big(\frac{1}{\delta}\big)}.
\end{align}
\begin{itemize}[leftmargin=*]
    \item Note that by Lemma 11 of \cite{abbasi2011improved} and the fact that $\lVert x_{t,a}\rVert_2\leq 1$ and $\lambda\geq 1$,
    \begin{equation}
        \sum_{t=1}^{T} {\lVert x_t \rVert}_{\inv{V_t}}^2 \leq \big(G_1(T)\big)^2.\label{eq:upper bound on sum of context norm square}
    \end{equation}
    Moreover, (\ref{eq:upper bound on sum of context norm square}) also implies that
    \begin{equation}
        \sum_{t=1}^{T} \alpha(t){\lVert x_t \rVert}_{\inv{V_t}}^2 \leq \alpha(T)\big(G_1(T)\big)^2. \label{eq:upper bound on sum of context norm square times alpha(t)}
    \end{equation}
    \item By combining (\ref{eq:upper bound on sum of context norm square}) and the Cauchy-Schwarz inequality, we have
    \begin{equation}
        \sum_{t=1}^{T} {\lVert x_t \rVert}_{\inv{V_t}} \leq \sqrt{T}\cdot G_1(T).\label{eq:upper bound on sum of context norm}
    \end{equation}
    
    \item By Lemma 3 in \cite{li2017provably} and since the minimum eigenvalue $\lambda_{\min}(V_t)\geq \lambda\geq 1$, for any $\delta\in [1/T,1)$, we know with probability at least $1-\delta$, the following result holds:
    \begin{equation}
        {\lVert \widehat{\theta}_t-\theta_* \rVert}_{V_t}\leq G_2(t,\delta), \hspace{6pt}\forall t\in \mathbb{N}.\label{eq:upper bound on norm of theta difference}
    \end{equation}
    \item By combining (\ref{eq:upper bound on sum of context norm}) and (\ref{eq:upper bound on norm of theta difference}), we thereby know that with probability at least $1-\delta$,
    \begin{equation}
        \sum_{t=1}^{T}{\lVert x_t \rVert}_{\inv{V_t}}\cdot \lVert {\widehat{\theta}_t-\theta_* \rVert}_{V_t} \leq \sqrt{T}\cdot G_1(T)G_2(T,\delta).
    \end{equation}
    \item Based on (\ref{eq:upper bound on norm of theta difference}), we further know that with probability at least $1-\delta$,
    \begin{equation}
        \sum_{t=1}^{T}\frac{\lVert {\widehat{\theta}_t-\theta_* \rVert}_{V_t}^2}{\alpha(t)}\leq \big(G_2(T,\delta)\big)^2\cdot \sum_{t=1}^{T}\frac{1}{\alpha(t)}.\label{eq:upper bound on norm of theta difference square over alpha(t)}
    \end{equation}
\end{itemize}
Summing up, by (\ref{eq:generalized regret 3-1})-(\ref{eq:upper bound on norm of theta difference square over alpha(t)}), the cumulative regret can be upper bounded as follows: With probability at least $1-\delta$,
\begin{align}
    \sum_{t=1}^{T}R_t \leq & T_0 +\sum_{t=T_0+1}^{T} C_1\alpha(t){\lVert x_t\rVert}_{\inv{V_t}}^2 \nonumber\\
    &+ C_2 {\lVert x_t \rVert}_{\inv{V_t}}\cdot \lVert {\widehat{\theta}_t-\theta_* \rVert}_{V_t}\nonumber\\
    &+ C_3\frac{\lVert {\widehat{\theta}_t-\theta_* \rVert}_{V_t}^2}{\alpha(t)}\\
    \leq &T_0+C_1\alpha(T)\big(G_1(T)\big)^2\nonumber\\
    &+ C_2 \sqrt{T}G_1(T)G_2(T,\delta)\nonumber\\
    &+C_3 \big(G_2(T,\delta)\big)^2\sum_{t=1}^{T}\frac{1}{\alpha(t)},
\end{align}
where $C_1:=\frac{2 L_{\mu}^4}{k_{\mu}^4}+\frac{1}{k_{\mu}^2}$, $C_2:=\frac{2L_{\mu}^3}{\kappa_{\mu}^2}+\frac{L_{\mu}}{\kappa_{\mu}}$, and $C_3:=\frac{L_{\mu}^2}{2}$.
Therefore, if $\alpha(t)=\Omega(\sqrt{t})$, then $\mathcal{R}(T)=\mathcal{O}(\alpha(T)\log T)$; Otherwise, if $\alpha(t)=\mathcal{O}(\sqrt{t})$, then $\mathcal{R}(T)=\mathcal{O}\big((\sum_{t=1}^{T}\frac{1}{\alpha(t)})\log T\big)$.
Hence, by choosing $\alpha(t)=\sqrt{t}$, we obtain a cumulative regret bound of $\mathcal{R}(T)=\mathcal{O}(\sqrt{T}\log T)$.\qed
\end{proof}

\section{Additional Experimental Results}
\label{appendix:experiments}
In this section, we present the additional experimental results for both linear bandits and the generalized case. Throughout the experiments, we set the random seed to be $46$.
\subsection{Linear Bandits}
To begin with, Tables \ref{table:TimeStatic/ID=9}, \ref{table:TimeVary/ID=2}, and \ref{table:TimeVary/ID=9} present the mean, standard deviation, and quantiles of the experiments described in Figures \ref{fig:regret}(b), \ref{fig:regret}(c), and \ref{fig:regret}(d), respectively.
Similar to what we observed from Table \ref{table:TimeStatic/ID=2}, LinRBMLE still exhibits better robustness than VIDS and most of the other benchmark methods under \rev{static contexts. Since the computation time is not unaffected by the values of the contexts, we only show the result of static contexts.}
\rev{Table \ref{table:std dev of computation time} shows the standard deviation of computation time for the results in Table \ref{table:TimeStatic_computation_time}. 
We observe that LinRBMLE is still among the best in standard deviation of computation time.}


\subsection{Generalized Linear Bandits}
For the generalized linear bandits, we perform a similar study on the effectiveness, efficiency, and scalability of GLM-RBMLE and the popular benchmark methods. 
The benchmark methods that are compared with GLM-RBMLE include UCB-GLM \cite{li2017provably} and Laplace-TS \cite{chapelle2011empirical} (Algorithm 3 in \cite{chapelle2011empirical}). The configurations of the three methods are as follows. 
We use $\alpha(t) = \sqrt{t}$,  $\eta(t) = 1 +\log t$, and $\lambda = 1$ for GLM-RBMLE, as suggested in Section \ref{section:generalized}.
Under UCB-GLM, after $\tau$ rounds of initial random selection, the arm with the largest $x_{t,a}^{\intercal} \widehat{\theta}_t+\chi \lVert x_{t,a}\rVert_{\inv{V_t}}$ is selected at each time $t$.
As suggested by \cite{li2017provably}, we take  $\chi=\frac{\sigma}{\kappa_{\mu}}\sqrt{\frac{d}{2}\log(1+2T/d)+\log(1/\delta)}$ with $\delta=0.1$, and let $\tau=K$. 
For Laplace-TS, we set the regularization parameter to be $1$.
Throughout the experiments of the generalized linear model, we consider the \textit{logistic} link function, i.e, $\mu(z)=1/(1+\exp(-z))$, for all $z\in\mathbb{R}$.
Similar to the experiments for LinRBMLE, for each comparison we consider both static contexts as well as time-varying contexts. The comparison contains 50 trials of experiments and $T$ rounds in each experiment. As the algorithms are computationally more intense for general linear bandits than for those for linear bandits, the time horizon is reduced to $T=10^3$ in the experiments for the generalized linear bandits.

\textbf{Effectiveness.}
Figure \ref{fig:GLM regret} and Tables \ref{table:GLM_ID_300_TimeStatic}-\ref{table:GLM_ID_1301_TimeVary} show the effectiveness of GLM-RBMLE in terms of cumulative regret. 
Under both static and time-varying contexts, GLM-RBMLE achieve the best mean regret performance in all the four configurations.
Similar to LinRBMLE, based on the results of standard deviation and regret quantiles, GLM-RBMLE also exhibits better robustness across sample paths than the two popular benchmark methods.
Specifically when contexts are static, GLM-RBMLE has lower standard deviation and $0.95$ quantile compare to UCB-GLM and Laplace-TS. We can characterize the statistical stability by standard deviation and quantiles so we give the result that GLM-RBMLE has better stability than others.
On the other hand, in Figure \ref{fig:GLM regret}, Laplace-TS appears to have not converged, but the corresponding regret quantiles provided by Tables \ref{table:GLM_ID_300_TimeStatic}-\ref{table:GLM_ID_1301_TimeVary} reveal that this is only because its performance in some trials is much worse than that in other trials.

\textbf{Efficiency.}
Figures \ref{fig:time all GLM} shows the averaged cumulative regret versus computation time per decision. We observe that GLM-RBMLE achieves the smallest average regret at the cost of a higher computation time compared to UCB-GLM. 

\textbf{Scalability.}
Table \ref{table:GLM scalability} presents computation time per decision as $K$ and $d$ are varied. We observe that under $K=5$ and $d=10,20,30$, the computation time per decision of GLM-RBMLE and UCB-GLM are comparable and much smaller than that of Laplace-TS. 
On the other hand, under $d=5$ and $K=10, 20, 30$, we also observe that the computation time of GLM-RBMLE is proportional to the number of arms, as indicated by Line 4 of Algorithm \ref{alg:GLM-RBMLE}.
It remains an interesting open question how to improve the scalability of GLM-RBMLE in terms of number of arms.

{\color{black}
\section{A Discussion on the Assumptions for GLM-RBMLE}
\label{appendix:discussion about assumption of GLM-RBMLE}
In the literature of the generalized linear bandit problems, a regret bound typically relies on either one of the following two sets of assumptions:

\noindent \underline{\textbf{1st Set of Assumptions}}:
\begin{itemize}[leftmargin=*]
    \item (1a) The $\ell_2$-norm of any context vector is upper bounded by some constant $M>0$ (Wlog, $M$ is chosen to be 1).
    \vspace{-1mm}
    \item (1b) The sequence of observed contexts is generated by an adversary (and hence not necessarilly i.i.d. across time). 
    \vspace{-1mm}
    \item (1c) The true parameter $\theta_*$ is in some closed bounded set $\Theta$ and hence $\theta_*$ is bounded (i.e. $\lVert \theta_*\rVert_2\leq S$, for some known positive constant $S$).
    \vspace{-1mm}
    \item (1d) The link function $\mu(\cdot)$ is continuously differentiable and is Lipschitz continuous with some constant $L_{\mu}$.
    \vspace{-1mm}
    \item (1e) The derivative of the link function $\mu(\cdot)$ satisfies a uniform property: $\inf_{\theta\in\Theta, \lVert x\rVert\leq 1}\mu'(\theta^{\intercal}x)>0$. 
    \vspace{-1mm}
\end{itemize}
The prior works that make the above set of assumptions include (Filippi et al., 2010; Zhang et al., 2016; Jun et al., 2017; Faury et al., 2020).

\noindent \underline{\textbf{2nd Set of Assumptions}}:
\begin{itemize}[leftmargin=*]
    \item (2a) The $\ell_2$-norm of any context vector is upper bounded by some constant $M>0$ (Wlog, $M$ is chosen to be 1).
    \vspace{-1mm}
    \item (2b) The observed contexts at each time $t$ are drawn i.i.d. from some distribution $\nu$. 
    \vspace{-1mm}
    \item (2c) The true parameter $\theta_*$ is in $\mathbb{R}^{d}$ but not necessarily in a closed bounded set.
    \vspace{-1mm}
    \item (2d) The link function $\mu(\cdot)$ is continuously differentiable and is Lipschitz continuous with some constant $L_{\mu}$.
    \vspace{-1mm}
    \item (2e) The derivative of the link function $\mu(\cdot)$ satisfies a local property: $\inf_{\lVert \theta-\theta_{*}\rVert_2 \leq 1, \lVert x\rVert\leq 1}\mu'(\theta^{\intercal}x)>0$. 
    \vspace{-1mm}
\end{itemize}
The prior works that make the above set of assumptions include (Li et al., 2017; Oh and Iyengar, 2019).

Note that the main differences between these two sets of assumptions are (1b), (2b), (1e), and (2e). Compared to (1e), the condition (2e) is more mild as it only requires that the derivative $\mu'(\cdot)$ is bounded for those $\theta$ close to $\theta_*$.
However, such relaxation is achieved at the expense of an additional i.i.d. assumption on the observed contexts (i.e. condition (2b)), which is required by the normality-type results of Maximum Likelihood Estimation (e.g. Proposition 1 in (Li et al., 2017)).

In this paper, we adopt the first set of assumptions and show that the proposed GLM-RBMLE achieves a regret bound of $\mathcal{O}(\sqrt{T}\log T)$.
As described in Section \ref{section:problem}, we consider the condition that 
\begin{equation}
    \kappa_{\mu}:=\inf_{z\in\mathbb{R}}\mu'(z)>0.\label{eq:kappa mu in appendix}
\end{equation}
Below we explain why the condition (\ref{eq:kappa mu in appendix}) holds without loss of generality under the first set of assumptions (1a)-(1e):
Given any link function ${\mu}(\cdot)$ that satisfies (1c)-(1e), we can construct a modified link function $\tilde{\mu}(\cdot):\mathbb{R}\rightarrow \mathbb{R}$ defined as 
\begin{equation}
    \tilde{\mu}(z):=
    \begin{cases}
        \mu(z) &,-S\leq z\leq S,\\
        \mu(S)+\mu'(S)(z-S) &, z>S,\\
        \mu(-S)+\mu'(-S)(z+S) &, z<-S.
    \end{cases}
\end{equation}
Hence, $\tilde{\mu}(\cdot)$ is constructed by first truncating the original link function $\mu(\cdot)$ and then extending the truncated function to the whole real line via linear extrapolation.
It is easy to verify that $\tilde{\mu}(\cdot)$ satisfies (\ref{eq:kappa mu in appendix}) under the condition (1e).
Moreover, as $\theta_*^{\intercal}x$ must be in $[-S,S]$ for any context $x$ under the assumptions (1a) and (1c), the above extension would not cause any model misspecification.
Therefore, given any standard link function ${\mu}(\cdot)$ (e.g. a logistic function), we can construct the corresponding $\tilde{\mu}(\cdot)$ through the above extension and use $\tilde{\mu}(\cdot)$ for the GLM-RBMLE.
Hence, the regret bound of GLM-RBMLE holds for the same class of link functions as the prior works (Filippi et al., 2010; Zhang et al., 2016; Abeille et al., 2017; Jun et al., 2017; Faury et al., 2020).





\section{A Discussion on the Computational Complexity}
\label{appendix:discussion about complexity}
In this section, we discuss the theoretical computational complexity of each benchmark method. Recall that $K$ is number of arms and $d$ is the dimension of context. For the simple index policies including LinRBMLE, LinUCB, and GPUCB, the per-decision complexity is $\smash[b]{\mathcal{O}(d^m + Kd^2)}$, where $d^m$ results from matrix inversion $V^{-1}_t$(with $m = 2.37\sim3$) and $Kd^2$ results from matrix multiplication $x_{t,a}^{\intercal}V^{-1}_tx_{t,a}$. LinTS has a per-decision complexity of $\smash[b]{\mathcal{O}(d^m + d^n + Kd^2)}$ with $n = 2.37\sim3$, where the additional term $d^n$ results from the sampling of a multivariate normal distribution. Bayes-UCB and KG have the same complexity: $\smash[b]{\mathcal{O}(d^m + Kd^2 + KS_1)}$, where $S_1$ results from the computing probability density function (PDF), cumulative distribution function (CDF), or the percent point function (i.e. inverse CDF) of a normal distribution. 
The effect of $S_1$ is empirically significant under large $K$'s (i.e. a large number of arms).
The complexity of VIDS is $\smash[b]{\mathcal{O}(d^m + MKd^2 + KS_2)}$, where $M$ is the number of posterior samples in Algorithm 6 of \cite{russo2018learning}, and $S_2$ is the time of solving the 1-dimensional optimization problem once (Line 1 of Algorithm 3 in \cite{russo2018learning}).
To achieve good regret performance, $M$ needs to be sufficiently large.
Moreover, the effect $S_2$ can be quite significant under large $K$'s.

\vspace{5mm}
\noindent\textbf{\Large References for the Appendix:}
\vspace{1mm}

\noindent Abbasi-Yadkori, Y.; P\`{a}l, D.; and Szepesv\`{a}ri, C. 2011. Improved algorithms for linear stochastic bandits. In \textit{Advances
in Neural Information Processing Systems}, 2312–2320.

\noindent Abeille, M.; Lazaric, A.; et al. 2017. Linear Thompson sampling revisited. \textit{Electronic Journal of Statistics} 11(2):5165–5197.

\noindent Bertsekas, D. P. 1999. \textit{Nonlinear programming}. Athena Scientific.

\noindent Chapelle, O.; and Li, L. 2011. An empirical evaluation of Thompson sampling. In \textit{Advances in neural information processing systems}, 2249–2257.

\noindent Faury, L.; Abeille, M.; Calauzènes, C.; and Fercoq, O. 2020. Improved Optimistic Algorithms for Logistic Bandits. In \textit{International Conference on Machine Learning}.

\noindent Filippi, S.; Cappe, O.; Garivier, A.; and Szepesv\`{a}ri, C. 2010. \textit{Parametric bandits: The generalized linear case. In Advances
in Neural Information Processing Systems}, 586–594.

\noindent Jun, K. S.; Bhargava, A.; Nowak, R.; and Willett, R. 2017. Scalable generalized linear bandits: Online computation and hashing. In \textit{Advances in Neural Information Processing Systems}, 99-109.

\noindent Li, L.; Lu, Y.; and Zhou, D. 2017. Provably optimal algorithms for generalized linear contextual bandits. In \textit{Proceedings of the 34th International Conference on Machine Learning-Volume} 70, 2071–2080. JMLR. org.

\noindent Oh, M. H.; Iyengar, G. 2019. Multinomial Logit Contextual Bandits. In \textit{Reinforcement Learning for Real Life (RL4RealLife) Workshop in International Conference on Machine Learning}.

\noindent Russo, D.; and Van Roy, B. 2018. Learning to optimize via
information-directed sampling. \textit{Operations Research} 66(1):
230–252.

\noindent Zhang, L.; Yang, T.; Jin, R.; Xiao, Y.; and Zhou, Z. H. 2016. Online stochastic linear optimization under one-bit feedback. In \textit{International Conference on Machine Learning}, 392-401.
}
\begin{table*}[!h]
\footnotesize
\begin{center}
\begin{tabular}{|c|c|c|c|c|c|c|c|c|c|}
\hline
\textbf{Alg.} & \textbf{RBMLE} & \textbf{LinUCB} & \textbf{BUCB} & \textbf{GPUCB} & \textbf{GPUCBT} & \textbf{KG} & \textbf{KG*} & \textbf{LinTS} & \textbf{VIDS} \\ \hline
Mean & 1.86 & 2.72 & 4.66 & 3.77 & {\ul \textbf{0.86}} & 19.14 & 2.81 & 13.49 & {\ul \textbf{0.83}} \\ \hline
Std.Dev & {\ul \textbf{0.45}} & 10.64 & 14.63 & 1.42 & {\ul \textbf{0.65}} & 35.38 & 8.37 & 2.10 & 1.30 \\ \hline
Q.10 & 1.48 & {\ul \textbf{0.05}} & 0.09 & 2.08 & 0.38 & {\ul \textbf{0.04}} & 0.09 & 10.51 & 0.21 \\ \hline
Q.25 & 1.63 & {\ul \textbf{0.06}} & 0.10 & 2.72 & 0.49 & {\ul \textbf{0.05}} & 0.12 & 12.23 & 0.30 \\ \hline
Q.50 & 1.77 & {\ul \textbf{0.12}} & 0.13 & 3.73 & 0.66 & {\ul \textbf{0.10}} & 0.16 & 13.70 & 0.43 \\ \hline
Q.75 & 1.99 & 0.36 & {\ul \textbf{0.27}} & 4.35 & 0.91 & 18.06 & {\ul \textbf{0.26}} & 14.92 & 0.55 \\ \hline
Q.90 & 2.39 & 2.83 & 5.64 & 6.06 & {\ul \textbf{1.64}} & 87.14 & 6.58 & 16.16 & {\ul \textbf{1.22}} \\ \hline
Q.95 & {\ul \textbf{2.55}} & 8.86 & 39.66 & 6.64 & {\ul \textbf{2.06}} & 100.66 & 19.38 & 16.64 & 4.57 \\ \hline
\end{tabular}
\caption{Statistics of the final cumulative regret in Figure \ref{fig:regret}(b). The best and the second-best are highlighted. ‘Q’ and “Std.Dev'' stand for quantile and standard deviation of the total cumulative regret over 50 trails, respectively. All the values displayed here are scaled by 0.01 for more compact notations. }
\label{table:TimeStatic/ID=9}
\end{center}
\end{table*}

\begin{table*}[!h]
\footnotesize
\begin{center}
\begin{tabular}{|c|c|c|c|c|c|c|c|c|c|}
\hline
\textbf{Alg.} & \textbf{RBMLE} & \textbf{LinUCB} & \textbf{BUCB} & \textbf{GPUCB} & \textbf{GPUCBT} & \textbf{KG} & \textbf{KG*} & \textbf{LinTS} & \textbf{VIDS} \\ \hline
Mean & 0.41 & 0.40 & {\ul \textbf{0.40}} & 0.52 & {\ul \textbf{0.38}} & 0.41 & 0.44 & 9.17 & 20.01 \\ \hline
Std.Dev & 0.17 & 0.19 & 0.15 & {\ul \textbf{0.14}} & 0.15 & {\ul \textbf{0.14}} & 0.16 & 0.25 & 0.65 \\ \hline
Q.10 & {\ul \textbf{0.25}} & 0.25 & {\ul \textbf{0.24}} & 0.38 & 0.25 & 0.26 & 0.28 & 8.87 & 19.20 \\ \hline
Q.25 & 0.30 & 0.28 & {\ul \textbf{0.27}} & 0.43 & {\ul \textbf{0.27}} & 0.32 & 0.35 & 9.01 & 19.61 \\ \hline
Q.50 & 0.37 & {\ul \textbf{0.35}} & 0.36 & 0.50 & {\ul \textbf{0.34}} & 0.38 & 0.40 & 9.15 & 20.05 \\ \hline
Q.75 & 0.47 & 0.47 & 0.49 & 0.56 & {\ul \textbf{0.44}} & {\ul \textbf{0.46}} & 0.53 & 9.25 & 20.35 \\ \hline
Q.90 & {\ul \textbf{0.57}} & 0.60 & 0.64 & 0.64 & {\ul \textbf{0.51}} & 0.61 & 0.63 & 9.49 & 20.75 \\ \hline
Q.95 & {\ul \textbf{0.63}} & {\ul \textbf{0.62}} & 0.69 & 0.72 & 0.65 & 0.71 & 0.70 & 9.73 & 21.09 \\ \hline
\end{tabular}
\caption{Statistics of the final cumulative regret in Figure \ref{fig:regret}(c). The best and the second-best are highlighted. ‘Q’ and “Std.Dev'' stand for quantile and standard deviation of the total cumulative regret over 50 trails, respectively. All the values displayed here are scaled by 0.01 for more compact notations.}
\label{table:TimeVary/ID=2}
\end{center}
\end{table*}

\begin{table*}[!h]
\footnotesize
\begin{center}
\begin{tabular}{|c|c|c|c|c|c|c|c|c|c|}
\hline
\textbf{Alg.} & \textbf{RBMLE} & \textbf{LinUCB} & \textbf{BUCB} & \textbf{GPUCB} & \textbf{GPUCBT} & \textbf{KG} & \textbf{KG*} & \textbf{LinTS} & \textbf{VIDS} \\ \hline
Mean & {\ul \textbf{0.40}} & 0.44 & 0.44 & 0.52 & {\ul \textbf{0.40}} & 0.45 & 0.46 & 9.48 & 19.85 \\ \hline
Std.Dev & 0.19 & 0.18 & 0.19 & {\ul \textbf{0.11}} & {\ul \textbf{0.12}} & 0.18 & 0.16 & 0.32 & 0.66 \\ \hline
Q.10 & {\ul \textbf{0.21}} & 0.25 & {\ul \textbf{0.23}} & 0.40 & 0.26 & 0.25 & 0.29 & 8.95 & 19.24 \\ \hline
Q.25 & {\ul \textbf{0.30}} & 0.32 & 0.30 & 0.43 & 0.31 & {\ul \textbf{0.29}} & 0.33 & 9.34 & 19.37 \\ \hline
Q.50 & {\ul \textbf{0.39}} & 0.43 & 0.41 & 0.53 & {\ul \textbf{0.39}} & 0.45 & 0.43 & 9.53 & 19.70 \\ \hline
Q.75 & {\ul \textbf{0.46}} & 0.53 & 0.57 & 0.62 & {\ul \textbf{0.48}} & 0.54 & 0.56 & 9.70 & 20.12 \\ \hline
Q.90 & {\ul \textbf{0.52}} & 0.60 & 0.62 & 0.67 & {\ul \textbf{0.55}} & 0.65 & 0.64 & 9.89 & 20.83 \\ \hline
Q.95 & {\ul \textbf{0.70}} & 0.77 & {\ul \textbf{0.70}} & {\ul \textbf{0.70}} & {\ul \textbf{0.59}} & 0.74 & 0.75 & 9.92 & 21.02 \\ \hline
\end{tabular}
\caption{Statistics of the final cumulative regret in Figure \ref{fig:regret}(d). The best and the second-best are highlighted. ‘Q’ and “Std.Dev'' stand for quantile and standard deviation of the total cumulative regret over 50 trails, respectively. All the values displayed here are scaled by 0.01 for more compact notations.}
\label{table:TimeVary/ID=9}
\end{center}
\end{table*}

\captionsetup[table]{labelfont={color=black}}
\begin{table*}[!h]
{\color{black}
\footnotesize
\begin{center}

\begin{tabular}{|c|c|c|c|c|c|c|c|c|c|}

\hline
\textbf{Algorithm} & \textbf{RBMLE} & \textbf{LinUCB} & \textbf{BUCB} & \textbf{GPUCB} & \textbf{GPUCBT} & \textbf{KG} & \textbf{KG*} & \textbf{LinTS} & \textbf{VIDS} \\ \hline
$d = 100$, $K = 100$ & 0.35 & 0.69 & 4.10 & 0.29 & 0.30 & 1.57 & 1.42 & 2.49 & 13.85 \\ \hline
$d = 200$, $K = 100$ & 0.50 & 0.84 & 3.83 & 0.47 & 0.47 & 1.49 & 1.21 & 3.12 & 40.98 \\ \hline
$d = 300$, $K = 100$ & 0.70 & 1.01 & 5.54 & 0.91 & 0.95 & 3.57 & 1.52 & 4.49 & 41.01 \\ \hline
$d = 100$, $K = 200$ & 0.65 & 0.54 & 8.05 & 0.75 & 0.96 & 4.73 & 2.98 & 2.87 & 7.91 \\ \hline
$d = 200$, $K = 200$ & 0.77 & 1.00 & 7.37 & 0.73 & 0.72 & 2.17 & 2.88 & 3.51 & 33.93 \\ \hline
$d = 300$, $K = 200$ & 3.28 & 4.10 & 22.43 & 4.05 & 4.07 & 18.65 & 7.95 & 4.12 & 15.05 \\ \hline
\end{tabular}

\caption{{\color{black}Standard deviation of computation time per decision for static contexts, under different values of $K$ and $d$. All numbers are averaged over 50 trials with $T= 10^2$ and in $10^{-4}$ seconds.}}
\label{table:std dev of computation time}
\end{center}
}
\end{table*}
\captionsetup[table]{labelfont={color=black}}

\begin{figure*}[!htp]
\center
$\begin{array}{c c c c}
    \multicolumn{1}{l}{\mbox{\bf }} & \multicolumn{1}{l}{\mbox{\bf }} & \multicolumn{1}{l}{\mbox{\bf }} & \multicolumn{1}{l}{\mbox{\bf }}\\ 
    \hspace{-5mm} \scalebox{0.263}{\includegraphics[width=\textwidth]{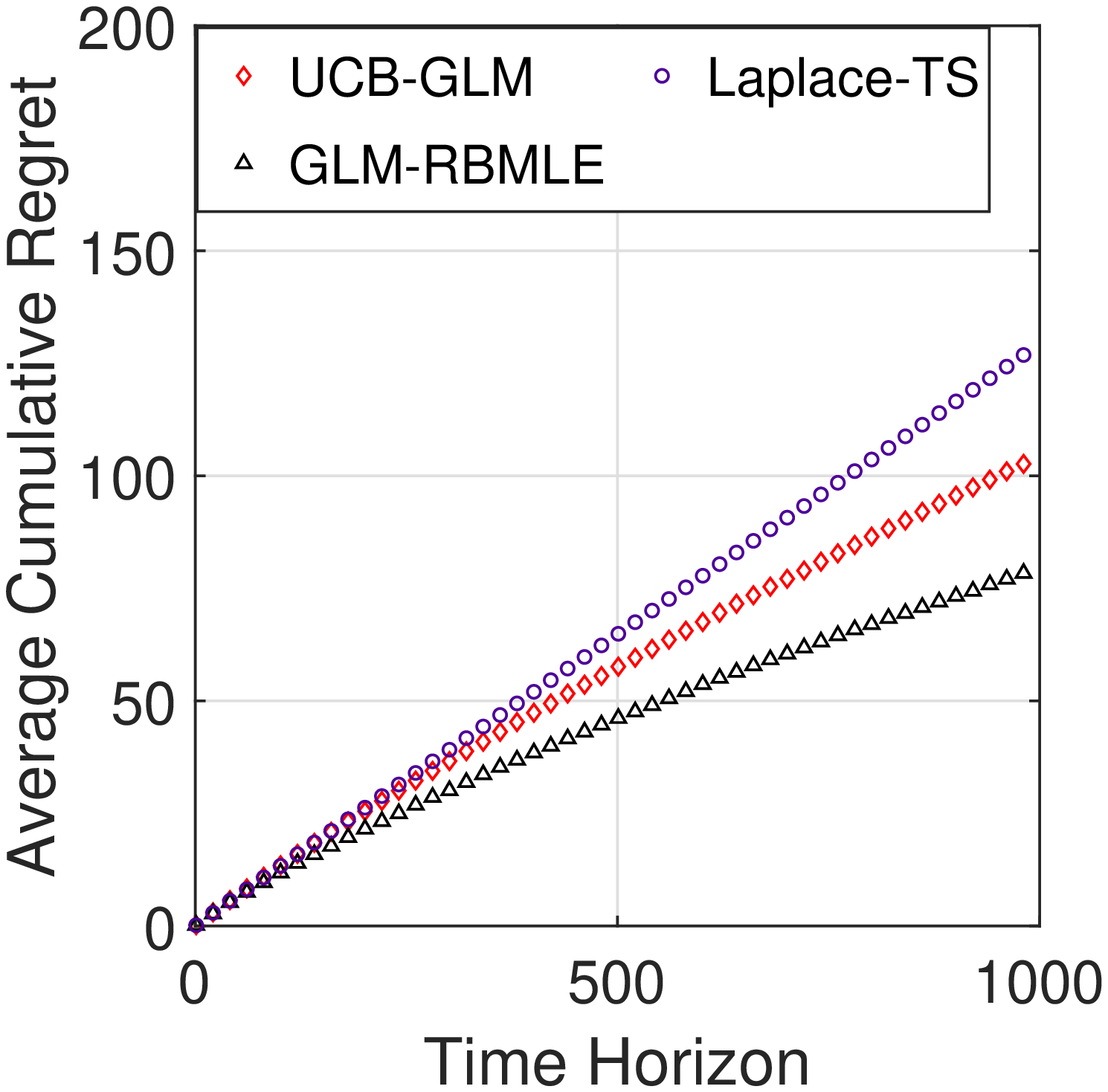}}  \label{fig:GLM_TimeStatic_ID_300} & \hspace{-4mm} \scalebox{0.263}{\includegraphics[width=\textwidth]{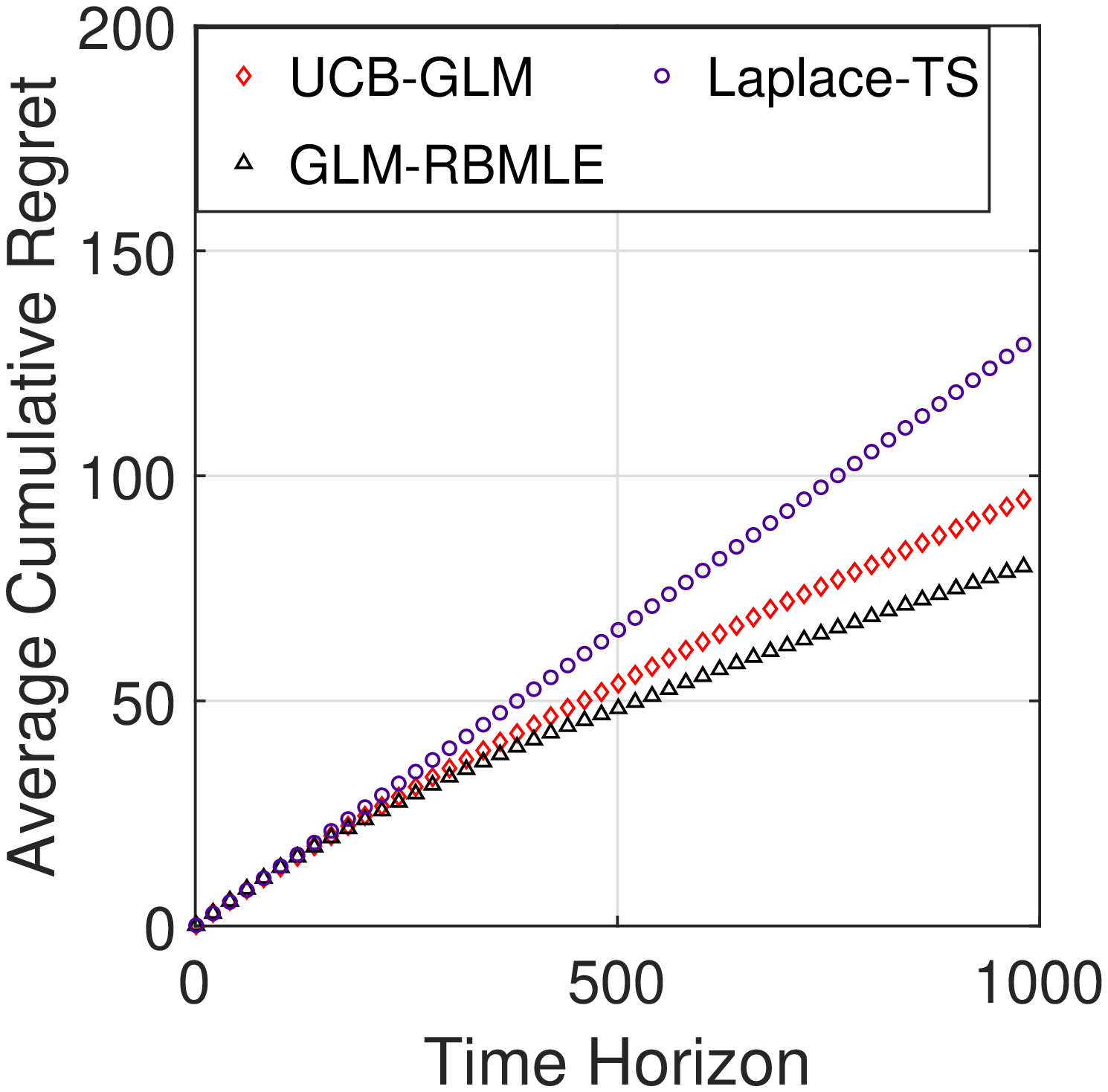}} \label{fig:GLM_TimeStatic_ID_301} & \hspace{-4mm}
    \scalebox{0.263}{\includegraphics[width=\textwidth]{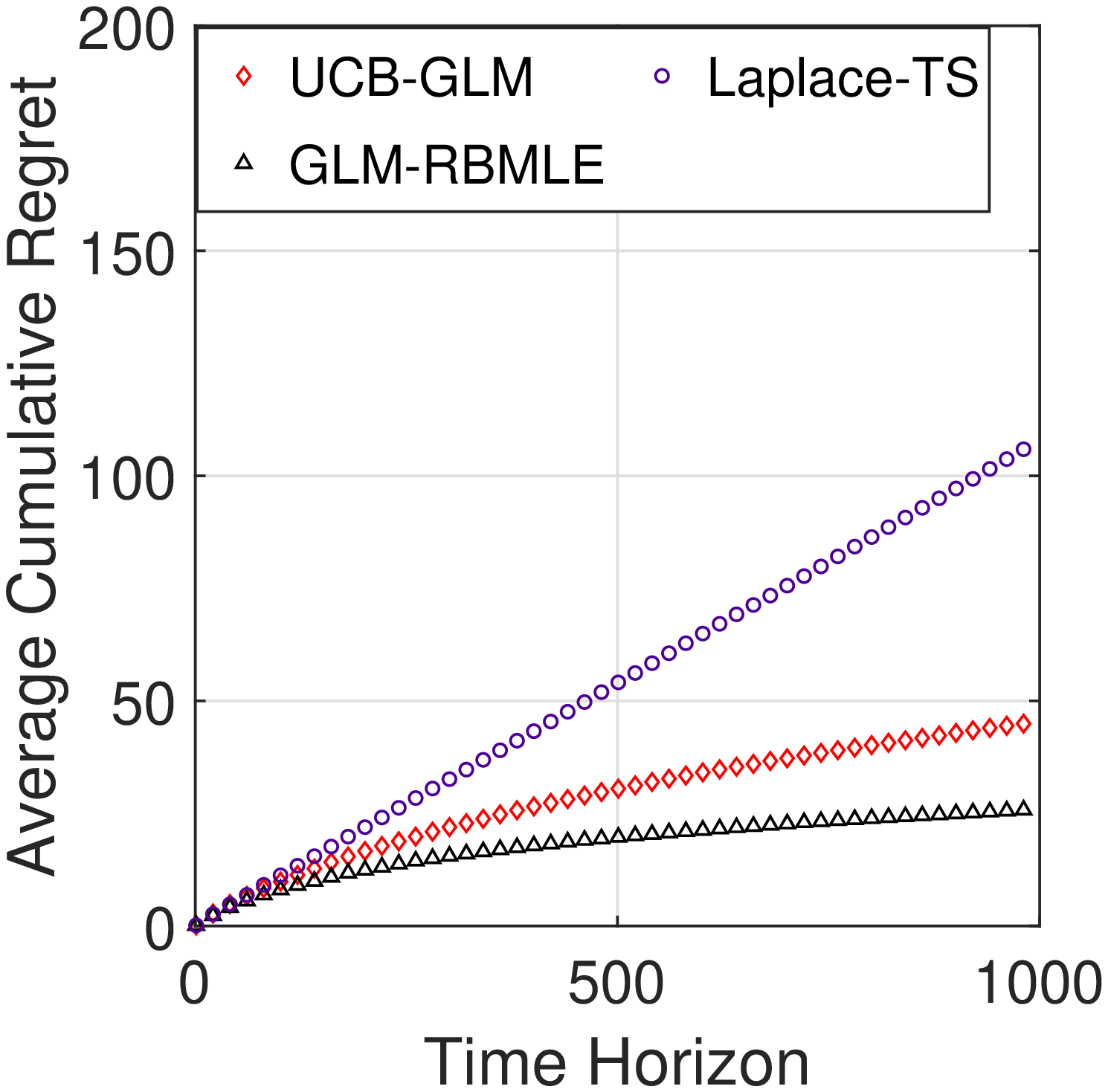}}  \label{fig:GLM_TimeVary_ID_1300}& \hspace{-4mm}
    \scalebox{0.263}{\includegraphics[width=\textwidth]{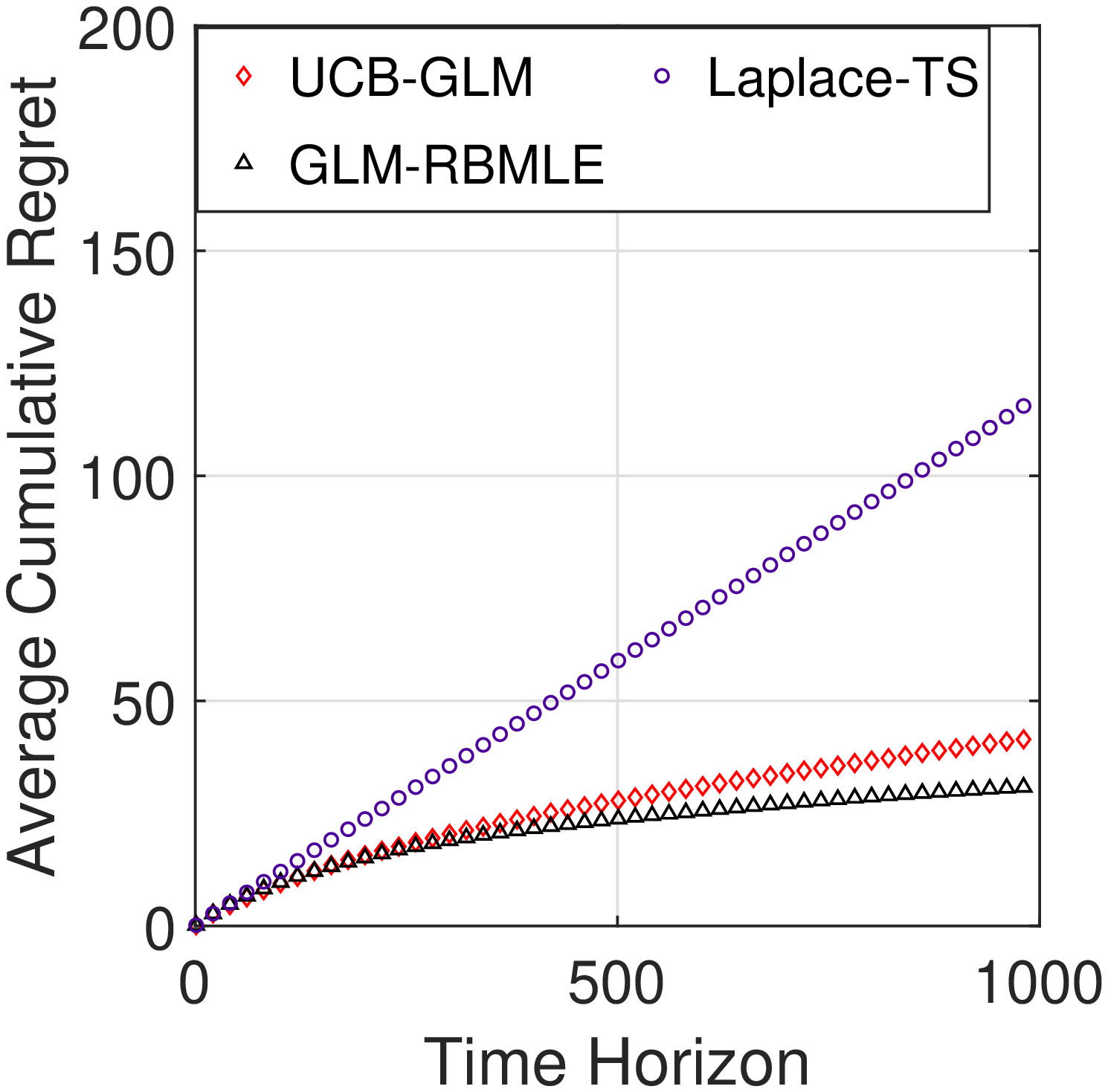}}  \label{fig:GLM_TimeVary_ID_1301} \\
    \mbox{\small (a)} & \hspace{-2mm} \mbox{\small (b)} & \hspace{-2mm} \mbox{\small (c)} & \hspace{-5mm} \mbox{\small (d)}\\[-0.2cm]
\end{array}$
\caption{Cumulative regret averaged over 50 trials with $T=10^3$ and $K=10$ on generalized linear bandits: (a) and
(b) are under static contexts; (c) and (d) are under time-varying contexts; (a) and (c) are with $\theta_* = (0.3, -0.5, 0.2,-0.7,-0.1)$; (b) and (d) are with $\theta_* = (0.2, -0.8, -0.5, 0.1, 0.1)$.}
\label{fig:GLM regret}
\end{figure*}

\begin{figure*}[!htp]
\center
$\begin{array}{c c c c}
    \multicolumn{1}{l}{\mbox{\bf }} & \multicolumn{1}{l}{\mbox{\bf }} & \multicolumn{1}{l}{\mbox{\bf }} & \multicolumn{1}{l}{\mbox{\bf }}\\ 
    \hspace{-5mm} \scalebox{0.268}{\includegraphics[width=\textwidth]{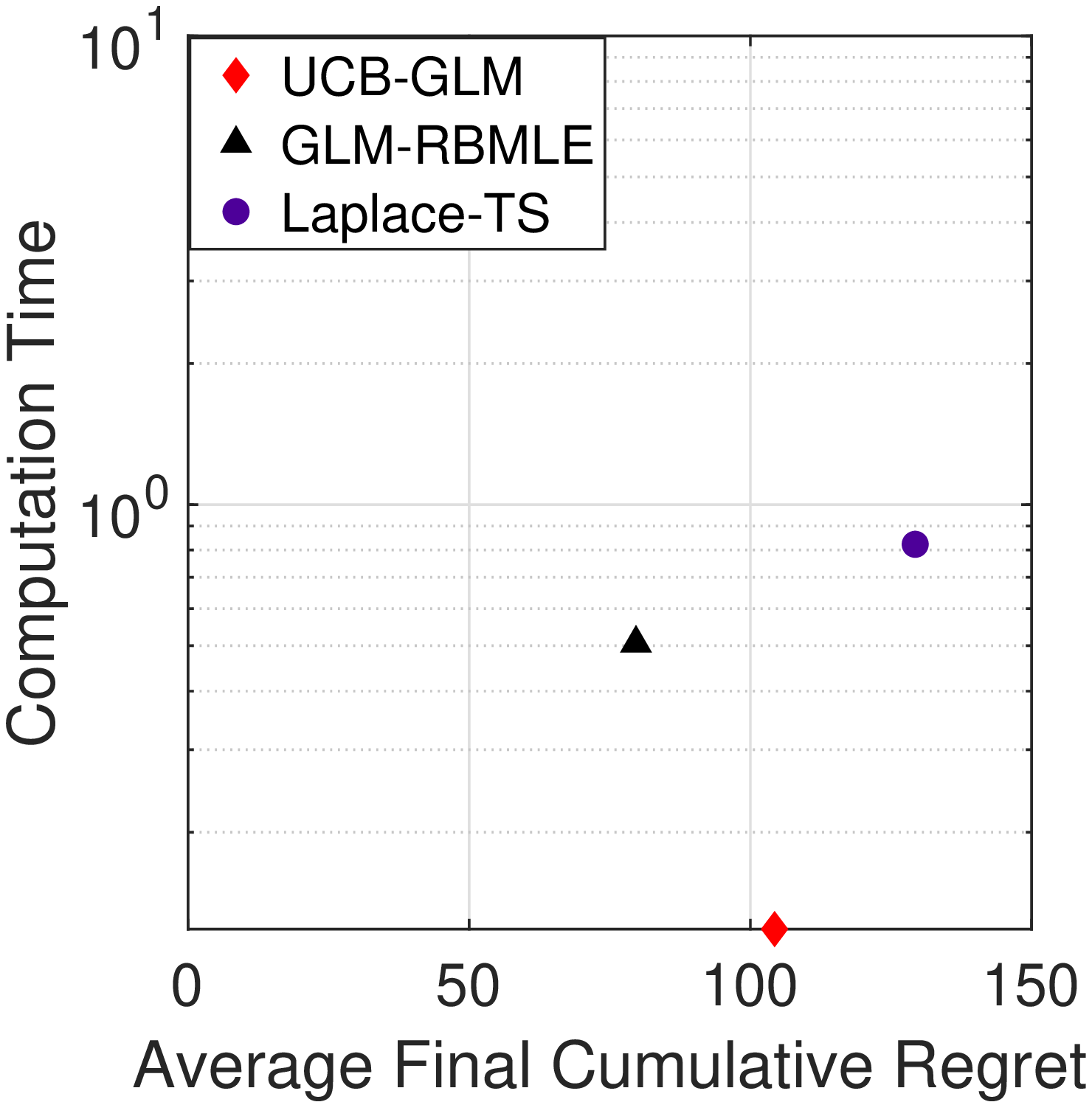}} \label{fig:GLM_Time_vs_Regret_TimeStatic_ID_300} & 
    \hspace{-5mm} \scalebox{0.268}{\includegraphics[width=\textwidth]{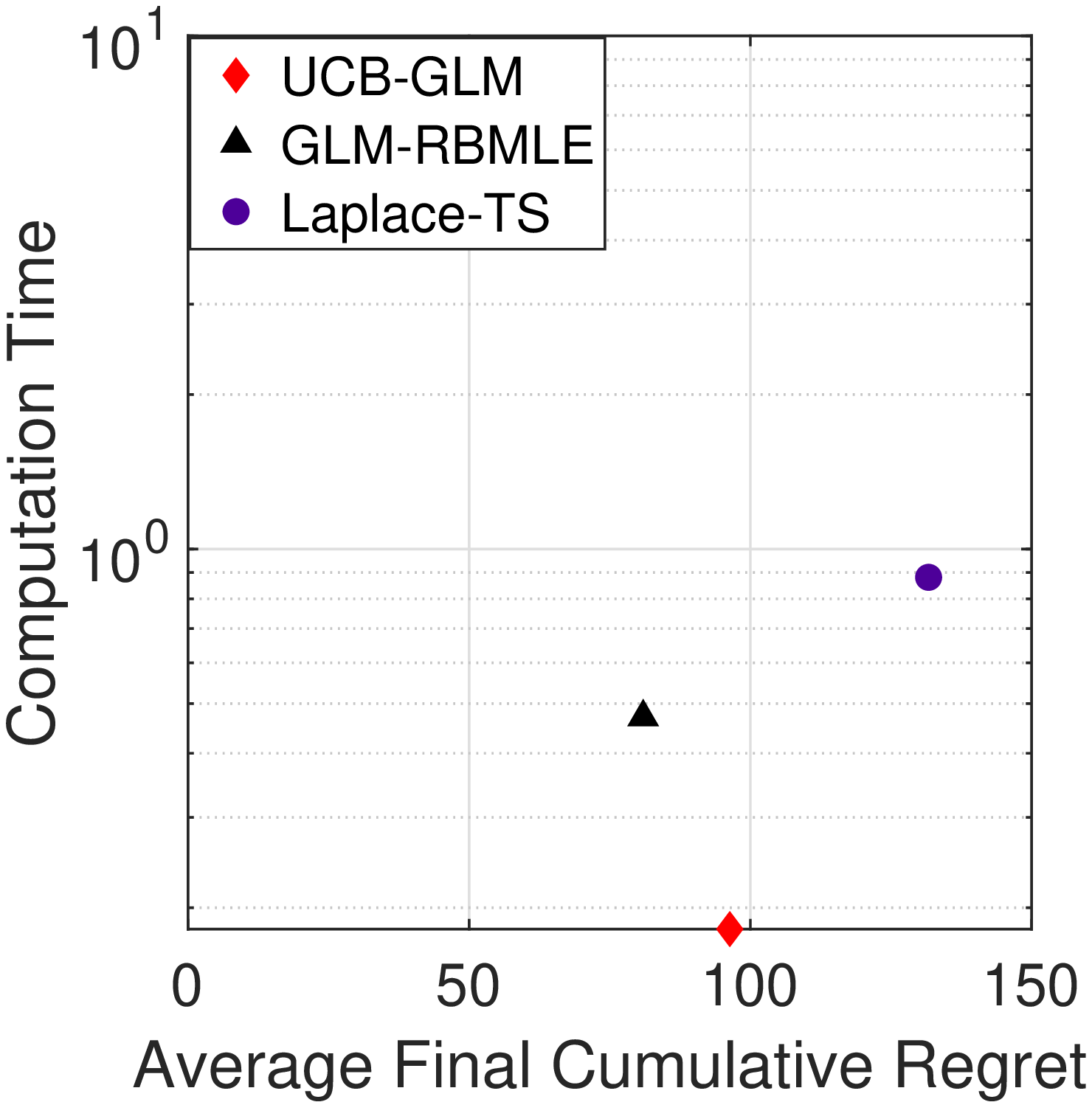}} \label{fig:GLM_Time_vs_Regret_TimeStatic_ID_301} &
    \hspace{-5mm} \scalebox{0.268}{\includegraphics[width=\textwidth]{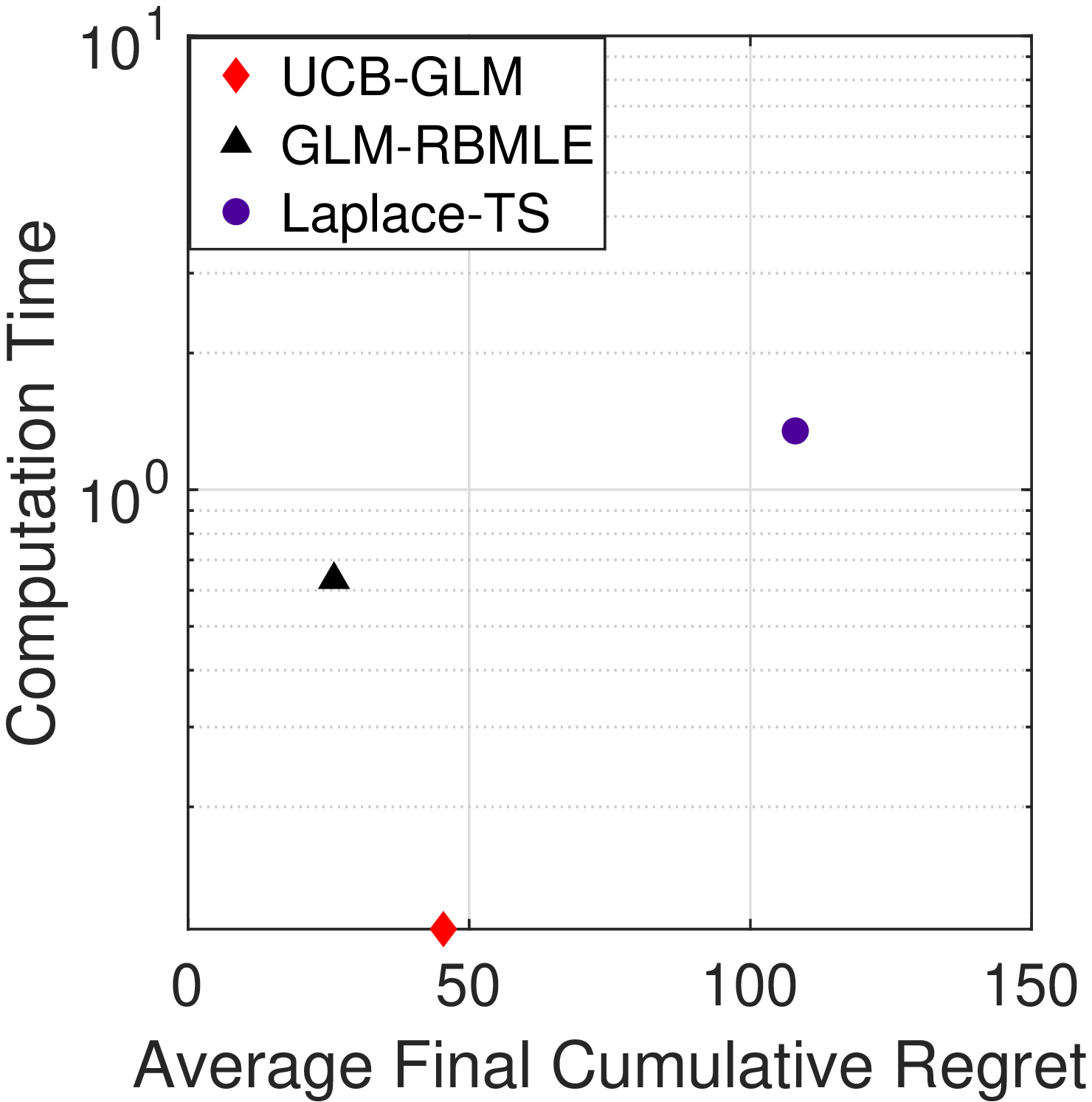}} \label{fig:GLM_Time_vs_Regret_TimeVary_ID_1300} &
    \hspace{-5mm} \scalebox{0.268}{\includegraphics[width=\textwidth]{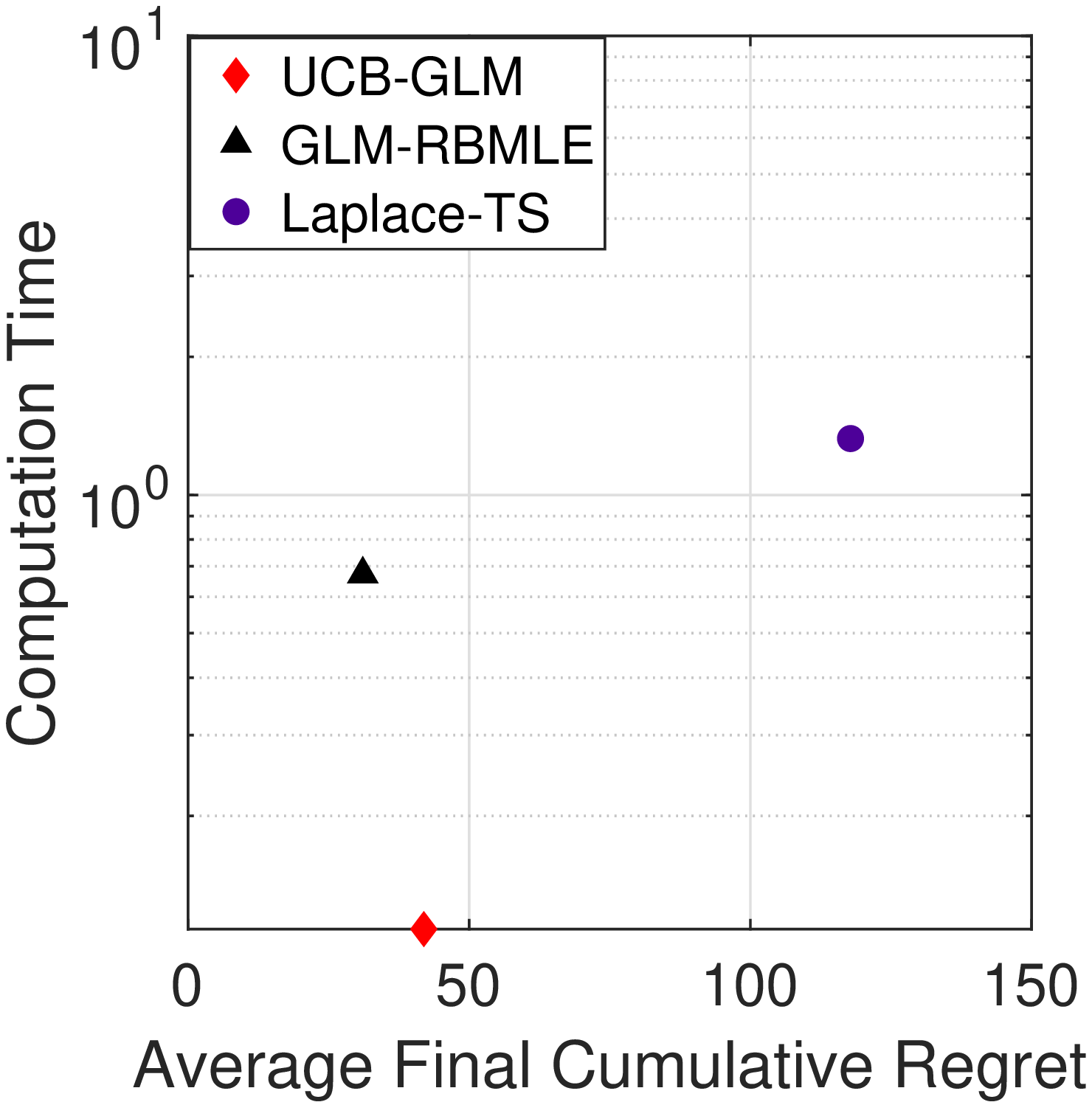}} \label{fig:GLM_Time_vs_Regret_TimeVary_ID_1301}\\ 
     \hspace{0mm}\mbox{\small (a)} & \hspace{-2mm} \mbox{\small (b)} & \hspace{-2mm} \mbox{\small (c)} & \hspace{-2mm} \mbox{\small (d)}\\[-0.2cm]
\end{array}$
\caption{Average computation time per decision vs. average final cumulative regret for (a) Figure \ref{fig:GLM regret}(a); (b) Figure \ref{fig:GLM regret}(b); (c) Figure \ref{fig:GLM regret}(c); (d) Figure \ref{fig:GLM regret}(d).}
\label{fig:time all GLM}
\end{figure*}


\begin{table*}[!h]
\center
\begin{tabular}{|c|c|c|c|}
\hline
\textbf{Algorithm} & \textbf{GLM-RBMLE} & \textbf{UCB-GLM} & \textbf{Laplace-TS} \\ \hline
Mean Final Regret & {\ul \textbf{79.66}} & 104.31 & 129.31 \\ \hline
Standard Deviation & {\ul \textbf{20.86}} & 31.52 & 87.92 \\ \hline
Quantile .10 & 55.53 & 69.60 & {\ul \textbf{11.65}} \\ \hline
Quantile .25 & 65.02 & 83.95 & {\ul \textbf{58.37}} \\ \hline
Quantile .50 & {\ul \textbf{78.56}} & 106.78 & 124.07 \\ \hline
Quantile .75 & {\ul \textbf{91.83}} & 125.10 & 197.74 \\ \hline
Quantile .90 & {\ul \textbf{106.03}} & 140.75 & 259.94 \\ \hline
Quantile .95 & {\ul \textbf{108.87}} & 153.24 & 264.79 \\ \hline
\end{tabular}
\caption{Statistics of the final cumulative regret in Figure \ref{fig:GLM regret}(a). The best one is highlighted.}
\label{table:GLM_ID_300_TimeStatic}
\end{table*}

\begin{table*}[!htbp]
\center
\begin{tabular}{|c|c|c|c|}
\hline
\textbf{Algorithm} & \textbf{GLM-RBMLE} & \textbf{UCB-GLM} & \textbf{Laplace-TS} \\ \hline
Mean Final Regret& {\ul \textbf{80.94}} & 96.34 & 131.69 \\ \hline
Standard Deviation & {\ul \textbf{25.38}} & 30.94 & 90.99 \\ \hline
Quantile .10 & 58.86 & 60.90 & {\ul \textbf{11.50}} \\ \hline
Quantile .25 & 63.85 & 72.74 & {\ul \textbf{53.86}} \\ \hline
Quantile .50 & {\ul \textbf{78.12}} & 95.25 & 125.30 \\ \hline
Quantile .75 & {\ul \textbf{92.96}} & 119.07 & 188.75 \\ \hline
Quantile .90 & {\ul \textbf{114.39}} & 131.07 & 248.53 \\ \hline
Quantile .95 & {\ul \textbf{131.95}} & 143.54 & 292.39 \\ \hline
\end{tabular}
\caption{Statistics of the final cumulative regret in Figure \ref{fig:GLM regret}(b). The best one is highlighted.}
\label{table:GLM_ID_301_TimeStatic}
\end{table*}

\begin{table*}[!htbp]
\center
\begin{tabular}{|c|c|c|c|}
\hline
\textbf{Algorithm} & \textbf{GLM-RBMLE} & \textbf{UCB-GLM} & \textbf{Laplace-TS} \\ \hline
Mean Final Regret& {\ul \textbf{25.95}} & 45.41 & 107.99 \\ \hline
Standard Deviation & 9.30 & {\ul \textbf{8.25}} & 57.90 \\ \hline
Quantile .10 & {\ul \textbf{15.92}} & 35.73 & 34.02 \\ \hline
Quantile .25 & {\ul \textbf{19.68}} & 38.57 & 65.03 \\ \hline
Quantile .50 & {\ul \textbf{23.11}} & 44.98 & 101.27 \\ \hline
Quantile .75 & {\ul \textbf{29.84}} & 51.50 & 145.02 \\ \hline
Quantile .90 & {\ul \textbf{35.71}} & 55.93 & 173.38 \\ \hline
Quantile .95 & {\ul \textbf{42.36}} & 60.32 & 213.75 \\ \hline
\end{tabular}
\caption{Statistics of the final cumulative regret in Figure \ref{fig:GLM regret}(c). The best one is highlighted.}
\label{table:GLM_ID_1300_TimeVary}
\end{table*}

\begin{table*}[!htbp]
\center
\begin{tabular}{|c|c|c|c|}
\hline
\textbf{Algorithm} & \textbf{GLM-RBMLE} & \textbf{UCB-GLM} & \textbf{Laplace-TS} \\ \hline
Mean Final Regret& {\ul \textbf{31.08}} & 41.93 & 117.81 \\ \hline
Standard Deviation & 13.40 & {\ul \textbf{6.50}} & 62.84 \\ \hline
Quantile .10 & {\ul \textbf{18.81}} & 34.87 & 32.58 \\ \hline
Quantile .25 & {\ul \textbf{21.64}} & 37.09 & 75.70 \\ \hline
Quantile .50 & {\ul \textbf{29.48}} & 41.97 & 119.50 \\ \hline
Quantile .75 & {\ul \textbf{36.09}} & 45.66 & 163.05 \\ \hline
Quantile .90 & {\ul \textbf{48.10}} & 51.15 & 203.09 \\ \hline
Quantile .95 & 55.04 & {\ul \textbf{54.06}} & 219.47 \\ \hline
\end{tabular}
\caption{Statistics of the final cumulative regret in Figure \ref{fig:GLM regret}(d). The best one is highlighted.}
\label{table:GLM_ID_1301_TimeVary}
\end{table*}

\begin{table*}[!h]
\center
\begin{tabular}{|c|c|c|c|}
\hline
\textbf{Algorithm} & \textbf{GLM-RBMLE} & \textbf{UCB-GLM} & \textbf{Laplace-TS} \\ \hline
$K=5, d = 10$ & 0.0275 & 0.0089 & 0.0675 \\ \hline
$K=5, d = 20$ & 0.0407 & 0.0216 & 0.2110 \\ \hline
$K=5, d = 30$ & 0.0519 & 0.0461 & 0.3691 \\ \hline
$K=10, d=5$ & 0.0406 & 0.0041 & 0.0305 \\ \hline
$K=20, d=5$ & 0.0823 & 0.0039 & 0.0331 \\ \hline
$K=30, d=5$ & 0.1225 & 0.0037 & 0.0333 \\ \hline
\end{tabular}
\caption{Average computation time per decision for static contexts in generalized linear bandit model, under different values of $K$ and $d$. All numbers are averaged over $50$ trials with $T= 10^2$ and in seconds. }
\label{table:GLM scalability}
\end{table*}

\end{document}